\documentclass[journal]{IEEEtran}
\usepackage{changes}

\usepackage{amsthm}

\usepackage{multirow}
\usepackage{booktabs}
\usepackage{comment}
\usepackage{bm}

\usepackage{times}
\usepackage{epsfig}
\usepackage{graphicx}
\usepackage{graphics}
\usepackage{amsmath}
\usepackage{amssymb}
\usepackage{bbm}
\usepackage{ulem}
\usepackage{enumitem}
\usepackage{multirow}
\usepackage{booktabs}
\usepackage{makecell}
\newcommand{\tabincell}[2]{\begin{tabular}{@{}#1@{}}#2\end{tabular}} 
\newtheorem{theorem}{Theorem}

\newtheorem{definition}[theorem]{Definition}

\newtheorem{remark}{Remark}

\numberwithin{theorem}{section}
\usepackage{listings}
\usepackage[export]{adjustbox}

\usepackage[framemethod=TikZ]{mdframed}
\usepackage[normalem]{ulem}
\usepackage{bbding}
\usepackage{float}
\usepackage{hyperref}
\usepackage{amssymb}

\ifCLASSINFOpdf
\else
\fi
\begin{document}
%
\title{Domain Invariant Model with Graph Convolutional Network for Mammogram Classification}
%
%
%

\author{Churan Wang,
        Jing Li,
        Xinwei Sun\Envelope{},
        Fandong Zhang,
        Yizhou Yu,~\IEEEmembership{Fellow,~IEEE},
        and Yizhou Wang 
\thanks{\Envelope{} indicates corresponding author}
\thanks{Churan Wang is with Center for Data Science, Peking University, Beijing, 100871, China, the work was done when she was an intern in Deepwise AI lab (e-mail: churanwang@pku.edu.cn).}
\thanks{Jing Li is with Dept. of Computer Science, Peking University, Beijing, 100871, China (e-mail: lijingg@pku.edu.cn).}
\thanks{Xinwei Sun is with Peking University, Beijing, 100871, China (e-mail: sxwxiaoxiaohehe@pku.edu.cn). }
\thanks{Fandong Zhang is with Center for Data Science, Peking University, Beijing, 100871, China (e-mail: fd.zhang@pku.edu.cn).} 
\thanks{Yizhou Yu is with Deepwise AI Lab, Beijing, 100080, China (e-mail: yizhouy@acm.org).}
\thanks{Yizhou Wang is with Dept. of Computer Science, Peking University, Beijing, 100871, China (e-mail: yizhou.wang@pku.edu.cn).}
}

%
%

\markboth{Churan Wang \MakeLowercase{\textit{et al.}}: Domain Invariant Model with GCN for Mammogram Classification}%
{Churan Wang \MakeLowercase{\textit{et al.}}: Domain Invariant Model with Graph Convolutional Network for Mammogram Classification}
%



\maketitle

\begin{abstract}  
Due to its safety-critical property, the image-based diagnosis is desired to achieve robustness on out-of-distribution (OOD) samples. A natural way towards this goal is capturing only clinically disease-related features, which is composed of macroscopic attributes (\emph{e.g.}, margins, shapes) and microscopic image-based features (\emph{e.g.}, textures) of lesion-related areas. However, such disease-related features are often interweaved with data-dependent (but disease irrelevant) biases during learning, disabling the OOD generalization. To resolve this problem, we propose a novel framework, namely \textbf{D}omain \textbf{I}nvariant \textbf{M}odel with \textbf{G}raph \textbf{C}onvolutional \textbf{N}etwork (DIM-GCN), which only exploits invariant disease-related features from multiple domains. Specifically, we first propose a Bayesian network, which explicitly decomposes the latent variables into disease-related and other disease-irrelevant parts that are provable to be disentangled from each other. Guided by this, we reformulate the objective function based on Variational Auto-Encoder, in which the encoder in each domain has two branches: the domain-independent and -dependent ones, which respectively encode disease-related and -irrelevant features. To better capture the macroscopic features, we leverage the observed clinical attributes as a goal for reconstruction, via Graph Convolutional Network (GCN). Finally, we only implement the disease-related features for prediction. 
The effectiveness and utility of our method are demonstrated by the superior OOD generalization performance over others on mammogram benign/malignant diagnosis. 
\end{abstract}


%
\IEEEpeerreviewmaketitle

\begin{table}[!t]
{
\resizebox{\linewidth}{!}{
\begin{tabular}{p{2.5cm}<{\centering}|p{1.2cm}<{\centering}|p{1.2cm}<{\centering}|p{1.2cm}<{\centering}|p{1.2cm}<{\centering}}
\hline
\multirow{2}{*}{Methodology} & \multicolumn{2}{c|}{ \tabincell{c}{test on \\ InH1}} & \multicolumn{2}{c}{ \tabincell{c}{test on \\ InH2}}\\
\cline{2-5}
 & \tabincell{c}{train on \\ InH2 \\ +InH3 \\ +DDSM} & \tabincell{c}{train on \\ InH1 } & \tabincell{c}{train on \\ InH1 \\ +InH3 \\ +DDSM} & \tabincell{c}{train on \\ InH2}\\
\hline
ERM~\cite{he2016deep} & 0.822 & 0.888 & 0.758 & 0.847 \\
\textbf{Ours} & \textbf{0.948} & \textbf{0.952} & \textbf{0.874} & \textbf{0.898} \\
\hline
\multirow{2}{*}{Methodology} & \multicolumn{2}{c|}{ \tabincell{c}{test on \\ InH3}} & \multicolumn{2}{c}{ \tabincell{c}{test on \\ DDSM} } \\
\cline{2-5}
 & \tabincell{c}{train on \\ InH1 \\ +InH2 \\ +DDSM} & \tabincell{c}{train on \\ InH3} & \tabincell{c}{train on \\ InH1 \\ +InH2 \\ +InH3} & \tabincell{c}{train on \\ DDSM} \\
\hline
ERM~\cite{he2016deep} & 0.735 & 0.776 & 0.779 & 0.847\\
\textbf{Ours} & \textbf{0.858} & \textbf{0.864} & \textbf{0.892} & \textbf{0.919}\\
\hline
\end{tabular}
}
}
\caption{AUC evaluation under out-of-distribution (OOD) and the same distribution (in-distribution). ERM defines training by Empirical Risk Minimization. For each testing set, the results on the left are under OOD circumstance, and the results on the right are under the same distribution.}
\label{tab:intro_compare}
\end{table}

\section{Introduction}

In medical diagnosis, a realistic issue that may hamper the clinical practice is: 
\emph{the distribution of data can vary greatly across healthcare facilities (\emph{e.g.}, hospitals)}, due to inconsistent imaging and preprocessing methods such as staining and scanning. This can fail the traditional Empirical Risk Minimization (ERM), as ERM heavily relies on \emph{independent and identically distributed (i.i.d)} assumption and can exploit spurious correlation during the data-fitting process. Such a spurious correlation may not generalize on unseen domains. This can be manifested by nearly 9\% AUC drop of ERM, as shown in Tab.~\ref{tab:intro_compare}. To satisfy the high safety requirements for medical diagnosis, it is desired to develop a model that can generalize well on out-of-distribution samples (\emph{i.e.}, distribute differently with training samples).

\begin{figure*}[t]
\begin{center}
    \includegraphics[height=0.55\linewidth]{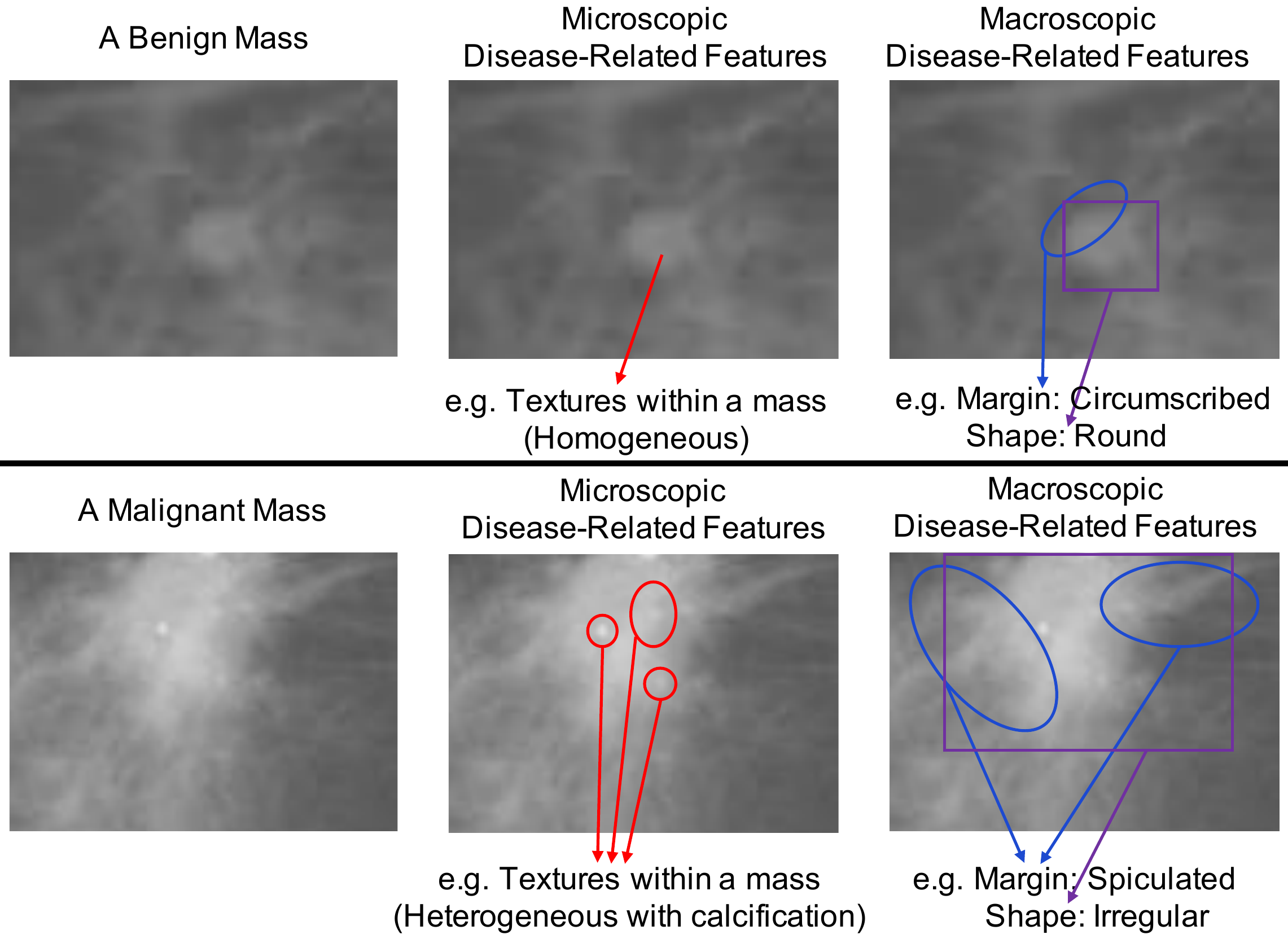}
\end{center}
    \centering
    \caption{The macroscopic and microscopic features of benign/malignant masses. The features behave differently between benign and malignant masses. Microscopic features are homogeneous for benign and heterogeneous for malignant (red arrow). For macroscopic features, the margin is clear and circumscribed in benign mass while spiculated in malignant mass (blue arrow). The shape is regular (\emph{e.g.,} round) in benign mass while irregular in malignant mass (purple arrow).
}
    \label{fig:overview}
\end{figure*}

Recently, there is increasing attention for OOD generalization, such as \cite{arjovsky2019invariant, ganin2016domain, li2018domain}. Given data from multiple domains, the above methods are proposed to learn the invariant representations for prediction. Such invariant learning can improve the generalization ability on general tasks. However, these methods fail to incorporate the medical priors about disease-related features and the underlying generating processes of them, which can limit the utility in medical diagnosis.

In clinical diagnosis, the disease-related features that doctors often employ for prediction are mainly composed of two parts: \emph{macroscopic} and \emph{microscopic} features. Specifically, the macroscopic features encode morphology-related attributes \cite{surendiran2012mammogram} of lesion areas, as summarized in American College of Radiology (ACR) \cite{sicklesacr}; the microscopic features, which reflect subtle patterns of lesions, are hard to observe but helpful for classifying the disease. 
Taking the mammogram in Fig.~\ref{fig:overview} as an illustration, the macroscopic features, e.g., can refer to the margins, shapes, spiculations of the masses; while the microscopic features, e.g., can refer to the textures, and the curvatures of contours~\cite{ding2020optimizing}. As these disease-related features can expose abnormalities (\emph{e.g.}, irregular shapes or textures) for malignancy compared to benign ones, they are determined by the disease status and provide invariant evidence for diagnosis. However, they are often mixed with other domain-dependent but disease-irrelevant noise, such as imaging and preprocessing biases, making them difficult to be identified during learning.

We encapsulate these priors into a latent generative model in Fig.~\ref{fig:graph}, in which the latent variables are explicitly split into three parts: the macroscopic features $a$, the microscopic features $s$ and other disease irrelevant features $z$. These three parts are modeled differently, such that \textbf{i)} as disease-related features, the $a$ and $s$ are invariantly related to the disease label, with $a$ additionally generating the observed attributes; while \textbf{ii)} other disease-irrelevant features $z$ are domain-dependent. We then prove that the disease-related features can be \emph{disentangled} from others. Under this guarantee, we propose a disentangling learning framework, dubbed as \textbf{D}omain \textbf{I}nvariant \textbf{M}odel with \textbf{G}raph \textbf{C}onvolutional \textbf{N}etwork (DIM-GCN), to only exploit disease-related features for prediction. Specifically, we design two-branch encoders for each domain: \emph{Relevant Enc} that is shared by all domains to learn disease-related features, and a domain-specific \emph{Irrelevant Enc} to extract disease-irrelevant features. To impose the disentanglement of invariant diseased-related features, among all latent features that participate in the reconstruction of the image, we only use the disease-related features for disease prediction. To additionally capture the macroscopic features, we enforce them to reconstruct the clinical attributes via Graph Convolutional Network (GCN).

To verify the utility and effectiveness of domain generalization, we perform our method on mammogram mass benign/malignant classification. Here the clinical attributes are those related to the masses, which are summarized in ACR \cite{sicklesacr} and easy to obtain. We consider four datasets (one public and three in-house) that are collected from different sources. In each time's evaluation, we train our method on three datasets and test on the remaining one. The results show that our method can outperform others by 6.2\%. Besides, our learned disease-related features can successfully encode the information on the lesion areas.

As a summary, our contributions are mainly three-fold: 
\begin{enumerate}
\item We propose a novel generative model, which incorporates medical priors regarding disease-related features;
\item We propose a novel DIM-GCN that can disentangle the disease-related features from others to prompt medical diagnosis on an unseen domain; 
\item Our model can achieve state-of-the-art OOD generalization performance for mass benign/malignant diagnosis.
\end{enumerate}

\section{Related Work}

\noindent\textbf{Patch-Level Mammogram Mass Classification.}
Previous approaches that can be used to address patch-level mammogram mass benign/malignant classification without ROI annotations are roughly categorized into three classes: (i) the GAN-based methods, \emph{e.g.,} Li \textit{et al.}~\cite{li2019signed}; (ii) the disentangling-based methods, \emph{e.g.,} Guided-VAE \cite{ding2020guided}; (iii) the attribute-guided methods, \emph{e.g.,} Chen \textit{et al.} \cite{chen2019multi}, ICADx \cite{kim2018icadx}. For class (i), they propose an adversarial generation to augment training data for better prediction. However, lacking the guidance of medical knowledge descents their performance. For class (ii), the disentangling mechanism can provide better disease-related representation learning but lacking the guidance of the prior of attributes limits their performance. For class (iii), the prior of attributes is considered into their methods. Insufficient utilization of attributes descents their effectiveness. Besides, above all methods do not consider domain bias while training. Changing the domain of data will directly cause drops on their results. Motivated by the above, we use the disentangling mechanism and domain knowledge with Graph Convolutional Network(GCN) for better learning invariant disease-related features and improving the ability of generalization in unseen domains.

\noindent\textbf{Domain Generalization.}
Domain generalization considers multiple domains and aims to improve the generalization performance in an unseen domain. For domain generalization, previous methods will lead to a dramatic performance decrease when testing on data from a different distribution with a different bias~\cite{ilse2020diva}. Thus such previous models are not robust enough to the actual task~\cite{azulay2018deep}. Progress has been made on domain generalization which can be roughly divided into three classes: (i) Learning the domain-specific constraints, \emph{e.g.,} \cite{chattopadhyay2020learning} aims to learn domain-specific masks to characterize different domains(\emph{e.g.,} clipart, sketch, and painting). They fail in medical images for masks are not suitable to distinguish different domains. (ii) Disentangle-based, \emph{e.g.,} \cite{ilse2020diva} proposes a generative model with three independent latent subspaces for the domain, the class and the residual variations respectively. They did not make use of the medical attribute knowledge which is important in our mammogram classification. (iii) Design invariant constrains, \emph{e.g.,} \cite{arjovsky2019invariant} aims to learn invariant representation across environments by minimizing the Invariant Risk Minimization term. \cite{ganin2016domain} and \cite{li2018domain} use adversarial way with the former performs domain-adversarial training to ensure a closer match between the source and the target distributions and the latter trains an Adversarial Autoencoder. Lack of disentanglement and the guidance of medical prior knowledge limits their performance on generalization.

\section{Methodology}

\noindent\textbf{Problem Setup $\&$ Notations.} Denote $x \in \mathcal{X}, y \in \mathcal{Y},A\in \mathcal{A}$ respectively as the image, benign/malignant label, and clinical attributes (\emph{e.g.}, margins, lobulations). We collect datasets $\{x^d_i,y^d_i,A^d_i\}$ from multiple domains (\emph{i.e.}, healthcare facilities in our scenario) $d \in \mathcal{D}$. Our goal is learning a disease-prediction model $f: \mathcal{X} \to \mathcal{Y}$ 
from training domains $\mathcal{D}_{\mathrm{train}}$, that generalizes well on all domains $\mathcal{D} \supset \mathcal{D}_{\mathrm{train}}$, especially out-of-distribution domains $\mathcal{D}_{\mathrm{test}} := \mathcal{D} \setminus \mathcal{D}_{\mathrm{train}}$. Denote $m:= |\mathcal{D}_{\mathrm{train}}|$.

This section is organized as follows: in Sec.~\ref{sec:DIM}, we first introduce our Bayesian network that encodes the medical prior knowledge of our problem. We prove that the disease-related features can be disentangled from others. With this theoretical guarantee, we in Sec.~\ref{sec:variational} introduce our learning method dubbed as \textbf{D}omain \textbf{I}nvariant \textbf{M}odel with \textbf{G}raph \textbf{C}onvolutional \textbf{N}etwork. Specifically, we reformulate the Variational Auto-Encoder (VAE) in Sec.~\ref{sec:reformulation}; then, we introduce our inference (encoder) model with disentanglement learning and generative (decoder) model with GCN in Sec.~\ref{sec:enc-dec}; finally, we introduce a \emph{variance regularizer} to further prompt the learning of invariant disease-related features in Sec.~\ref{sec:var}.



\begin{figure}[t]
\centering
\includegraphics[width=.95\linewidth]{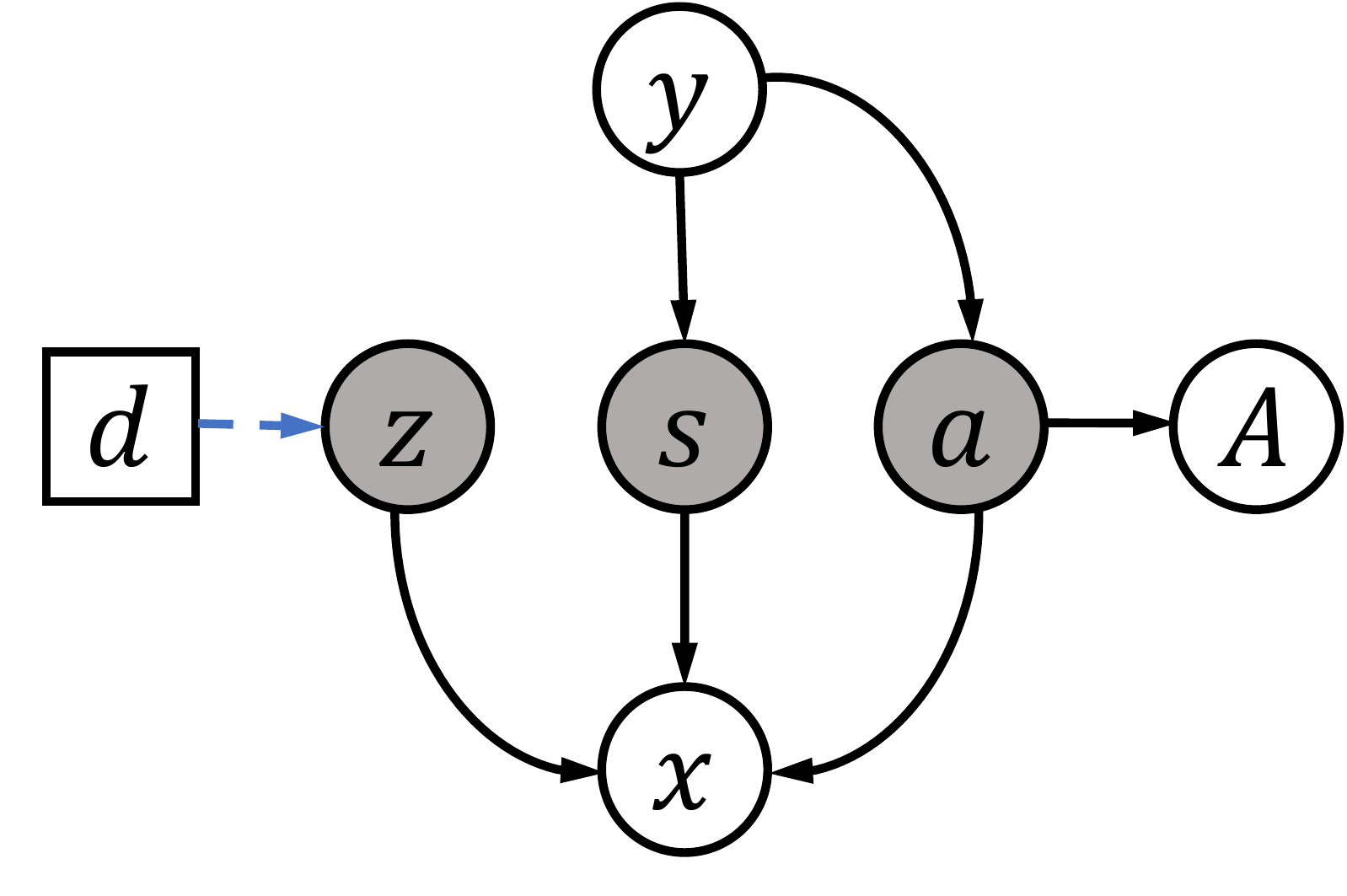}
\caption{Our Bayesian Network. The Bayesian network depicts the underlying generation process of the medical disease. The edges represent the probabilistic relationships between images, attributes, disease labels and domains in our problem.} 
\label{fig:graph}
\end{figure} 


\subsection{Bayesian Network with Latent Variables} 
\label{sec:DIM}



Our Bayesian Network is illustrated in Fig.~\ref{fig:graph}, which encodes the medical priors of disease diagnosis into the generating processes. Specifically, among all latent components that generate the image $x$, we split them into three parts: $a$, $s$ and $z$ that respectively encodes the information of macroscopic (such as shapes, margins \cite{sicklesacr}), microscopic (textures, curvatures of contours \cite{ding2020optimizing}) and disease-irrelevant features. As disease-related features that provide stable evidence for diagnosis, the $a$ and $s$ are assumed to generate from the disease status $y$ via an invariant mechanism ($p(s,a|y)$ is invariant across domains). Such features are often mixed with other variations, \emph{a.k.a.} $z$, which encodes the domain-dependent information such as imaging and pre-processing biases.
This mixture can lead to the difficulty of identifying the $s,a$ for prediction. To further differentiate $a$ from $s$, we additionally assume that the $a$ generates the clinical attributes $A$, which are often employed by clinicians for prediction due to its explainability and easiness to observe. Such disease-related macroscopic attributes $A$ are objectively existing properties of the lesion area. These properties and their corresponding lesion area, are generated concurrently from the disease status $y$, via the latent component $a$ \cite{bowyer1996digitalshort}. Our Fig.~\ref{fig:graph} depicts the underlying generation process of the medical image and clinical attributes, instead of the disease inference process on the observational data. The microscopic features $s$, which is hard to observe, can provide additional information for prediction. We assume the generating processes from $a \to A$ and $z,s,a \to X$, as they follow from physical laws, to be invariant across domains \cite{sun2020latent}. 

A natural \emph{identifiability} question towards robust prediction imposes: \emph{will the disease-related features (that estimated from $x$) can be possibly identified, without mixing the information of others?} The following theorem provides a positive answer, which provides a guarantee for us to propose the learning method that can learn the $p^d(x,y,A)$ well.

\begin{theorem}[Informal]
\label{thm:dis}
Suppose that multiple environments are diverse enough, then there exist functions $h_s,h_z,h_a$ such that, for any latent variables $(\tilde{s},\tilde{a},\tilde{z})$ and $(s^\star,a^\star,z^\star)$ giving rise to the same observational distribution (\emph{i.e.}, $p(x,y,A)$), we have that the  
\begin{align*}
    \tilde{s} = h_s(s^\star), \ \tilde{z} = h_z(z^\star), \ \tilde{a} = h_a(a^\star).
\end{align*}
\end{theorem}

\begin{figure*}[t]
\centering
\includegraphics[width=1.0\linewidth]{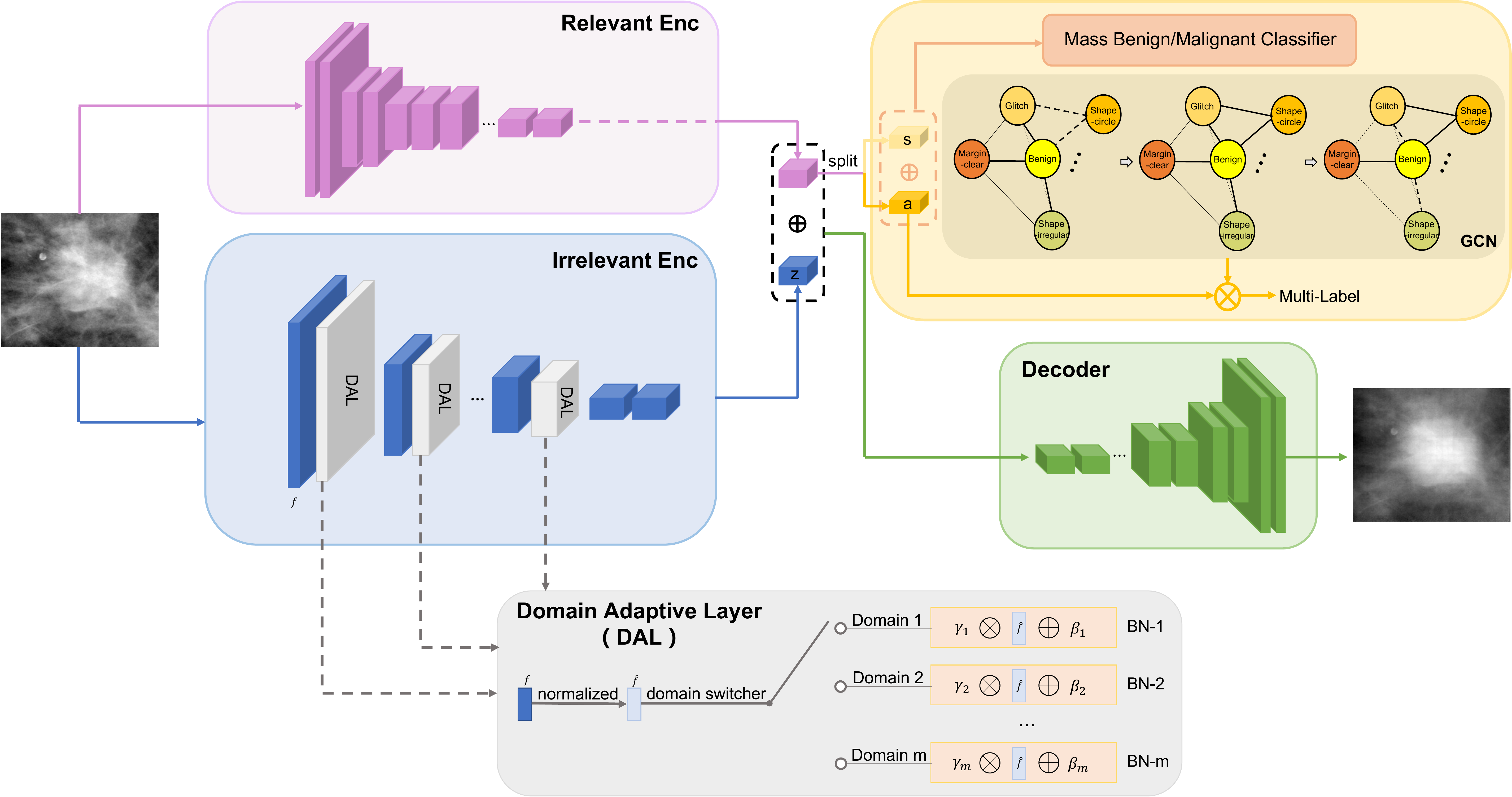}
\caption{Domain Invariant Model with Graph Convolutional Network (DIM-GCN). In the training phase, the input image is fed into \emph{Relevant Enc} and \emph{Irrelevant Enc}. In \emph{Irrelevant Enc}, images from different domains are fed into corresponding domain adaptive layer respectively. We optimize the overall network on multiple training domains. In the test phase, an image from an unseen domain is only fed into \emph{Relevant Enc} to get the disease-related factors $s$ and $a$. The final prediction result is obtained by the classifier.} 
\label{fig:network}
\end{figure*}

\begin{remark}
It can be easily seen from Theorem~\ref{thm:dis} that if $(s^\star,a^\star,z^\star)$ denote the ground-truth latent variables that generate $p(x,y,A)$, then the learned latent variables $\tilde{s}$, $\tilde{a}$ and $\tilde{z}$ \emph{do not} mix information of others. The diversity condition, which requires that the extent of dependency of $y$ on $z$ varies across domains, is almost necessary to ensure the invariance disease-related features to be identified, and is thus similarly assumed in the literature \cite{sun2020latent}.
\end{remark}



\subsection{Domain Invariant Model with Graph Convolutional Neural Network}
\label{sec:variational}
Guided by Theorem~\ref{thm:dis}, to achieve disentanglement, it is sufficient to learn $\{p^d(x,y,A)\}$ well. To achieve this goal, we first reformulate the \textbf{E}vidence \textbf{L}ower \textbf{BO}und (ELBO) of VAE based on Fig.~\ref{fig:graph} in Sec.~\ref{sec:reformulation}, followed by inference (\emph{a.k.a} Encoder) and generative (\emph{a.k.a} Decoder) models in Sec.~\ref{sec:enc-dec} and finally the \emph{variance regularizer} in Sec.~\ref{sec:var}. 
\subsubsection{\textbf{ELBO Reformulation}}
\label{sec:reformulation}
The VAE \cite{kingma2013auto} was proposed to learn $p(x)$ with $Z \to X$. To resolve the intractability for high-dimensional data $x$, it introduced the variational distribution $q(z|x)$ and to maximize the \textbf{E}vidence \textbf{L}ower \textbf{BO}und (ELBO) defined as $\max_{q,p_\theta} \mathbb{E}_{p(x)}\left( \mathbb{E}_{q(z|x)}\left( \log{\frac{p_\theta(x,z)}{q(z|x)}} \right) \right) \leq  \mathbb{E}_{p(x)}\left(\log{p_\theta(x)} \right)$. The ``=" can be achieved as long as $q(z|x) = p_\theta(z|x)$. In other words, the variational distribution can learn the behavior of the posterior distribution during optimization. 

To learn $p^d(x,y,A)$ following from Fig.~\ref{fig:graph} for each domain $d$, we similarly derive the ELBO for 
\begin{align}
    \log p_\theta^d(x, y, A) & =  \log p_\theta^d(x) + \log p_\theta^d(A|x) + \log p_\theta^d(y|x, A), \nonumber 
\end{align}
according to the Bayesian formula. We introduce the variational distribution $q^d(v|x)$ ($v := (z,s,a)$). The corresponding ELBO for the first term, \emph{i.e.}, $\log p^d_\theta(x)$ can be derived as:
\begin{align}
& \mathbb{E}_{p^d(x,y,A)}\left( \mathbb{E}_{q^d(v|x)}\left(\log{\frac{p^d_\theta(x,v)}{q^d(v|x)}}\right) \right) \label{eq:elbo_main} \\
= & \mathbb{E}_{p^d(x,y,A)}\left( \mathbb{E}_{q^d(v|x)}\left(\log{p_\theta(x|v)}\right) \right) \nonumber \notag \\
& \ \ \ \ \ - \mathbb{E}_{p^d(x,y,A)}\left(\mathrm{KL}(q^d(v|x), p_\theta^d(v)) \right), \notag \nonumber
\end{align}
which is $\leq \mathbb{E}_{p^d(x,y,A)} \left( \log{p^d_\theta(x)} \right)$ and the equality can only be achieved once $q^d(v|x)=p_\theta^d(v|x)$. Combined with $\log p^d_\theta(A|x)$ and $\log p^d_\theta(y|x, A)$, the loss function for domain $d$ is:
\begin{align}
    & \ \ell^d(q^d,p_\theta^d) = -\frac{1}{n_d} \sum_{i=1}^{n_d} \left( \log p^d_\theta(A_i|x_i) + \log p^d_\theta(y_i|x_i, A_i) + \right. \nonumber \\
    & \ \ \ \ \ \ \left. \mathbb{E}_{q^d(v|x_i)}\left(\log{p_\theta(x_i|v)}\right) - \mathrm{KL}(q^d(v|x_i), p_\theta^d(v)) \right). \nonumber 
\end{align}
With $q^d(v|x)$ learned to approximate $p^d_\theta(v|x)$, the $p^d_\theta(A|x)$ and $p^d_\theta(y|x, A)$ can be approximated via: 
\begin{align}
    p^d_\theta(A|x) & \approx \int p_\theta(A|a)q(a|x)da, \label{eq:A} \\
    p^d_\theta(y|A,x) & \approx \int p_\theta(y|s,a)q(s,a|x)dads.  \label{eq:y}
\end{align}
The $p_\theta(x|v)$ in Eq.~\eqref{eq:elbo_main} and $p_\theta(A|a)$, $p_\theta(y|s,a) = p_\theta(s,a|y)p_\theta(y)/p_\theta(s,a)$ in Eq.~\eqref{eq:A},~\eqref{eq:y}, which are invariant across domains, have their parameters $\theta$ shared by all domains $d$. To optimize the loss, we need to respectively parameterize the prior models $p^d_\theta(z,s,a)$, inference models $q^d(z,s,a|x)$ and generative models $p_\theta(x|z,s,a),p_\theta(A|a), p_\theta(y|s,a)$. 

Following the graph in Fig.~\ref{fig:graph}, the $p_\theta^d(z,s,a)$ can be factorized as $p_\theta^d(z,s,a)=p(s,a)p_\theta(z|d)$, where the $p(s,a)$ can be modeled as isotropic Gaussian while $p_\theta(z|d)$ can be learned via Multilayer Perceptron (MLP) with one-hot encoded vector $d \in \mathbb{R}^m$ as input. 

\noindent \textbf{Inference models.} We adopt the mean-field approximation to factorize $q(z,s,a|x,d)$ as $q(s,a|x)*q(z|x,d)$. This motivates us to implement a two-branch encoder, a domain-invariant (\emph{a.k.a}, relevant encoder) one $q(s,a|x)$ and a domain-specific one $q(z|x,d)$ (\emph{a.k.a}, irrelevant encoder), as shown in Fig.~\ref{fig:network} and the subsequent section. Together with prior models, the inference models are the inputs of KL-divergence term. 

\noindent \textbf{Generative models.} We optimize to reconstruct $x,A$ and predict $y$ via $p_\theta(x|z,s,a),p_\theta(A|a)$ and $p_\theta(y|s,a)$. Particularly, to model the correlation among attributes, we implement Graph Convolutional Network (GCN) to learn $p_\theta(A|a)$. 

As illustrated in Fig.~\ref{fig:network}, all models are optimized following a variational Auto-Encoder scheme. In the next subsection, we will introduce the architectures of the encoder, decoder to implement the inference models and the generative models.

\subsubsection{\textbf{Encoder-Decoder Architecture}}
\label{sec:enc-dec}
As shown in Fig.~\ref{fig:network}, our DIM-GCN contains the following components: \emph{two-branch encoders} dubbed as \emph{Relevant Enc} for $q(s,a|x)$ and \emph{Irrelevant Enc} for $q(z|x,d)$ to respectively encode the disease-related (\emph{i.e.}, $s,a$) and -irrelevant features (\emph{i.e.}, $z$), a \emph{decoder} for $p_{\theta}(x|v)$ to reconstruct the image $x$, a \emph{GCN} for $p_{\theta}(A|a)$ to reconstruct the attributes, and finally a \emph{disease classifier} for $p_{\theta}(y|s, a)$ for final disease classification. Each component is introduced in details below. 

\noindent \textbf{Two-Branch Encoders for $q(s, a|x)$ and $q^d(z|x)$.}
Motivated by the factorization of $q^d(z,s,a|x)$ into $q(s,a|x)*q(z|x,d)$ in Sec.~\ref{sec:DIM}, we propose two-branch encoders, namely \emph{Relevant Enc} for $q(s,a|x)$ and \emph{Irrelevant Enc} for $q(z|x,d)$. For the disease classification $p(y|s,a)$, the \emph{Relevant Enc} encodes the disease-related features $(s,a)$ from $x$, with the macroscopic features $a$ additionally enforced to reconstruct the attributes $A$ well. The \emph{Irrelevant Enc} encodes other disease-irrelevant features $z$ from $x$.
As such features are domain-specific, we incorporate a domain adaptive layer into the encoder. Specifically, the domain adaptive layer that following the convolutional layer contains m batch normalization (BN) layers, as shown in Fig.~\ref{fig:network}. The scaling and shifting parameters in each BN are learned adaptively with the domain in the training.

Denote the features from the convolutional layer as $f$, then we get the output features of the domain adaptive layer for each domain $d$ ($d \in \{1, ..., m\}$) by:
\begin{equation}
\begin{aligned}
    f_1 = BN_{\gamma_1, \beta_1}(f) 
    & = \gamma_1 * \hat{f} + \beta_1 \\
    \vdots \\
    f_m = BN_{\gamma_m, \beta_m}(f) 
    & = \gamma_m * \hat{f} + \beta_m 
\end{aligned}
\end{equation}
with $\hat{f} = \frac{f-\mu_B}{\sqrt{\delta_B^2 + \epsilon}}$ is the normalized features by the mini-batch mean $\mu_B$ and variance $\delta_B$. Each domain has a set of learning parameters $\gamma_d, \beta_d$.

Such a two-branch encoder is trained by a KL-divergence term:
\begin{align}
     & \mathbb{E}_{p^d(x)} \left( \mathrm{KL}(q^d(z,s,a|x),p^d(z,s,a)  \right) \nonumber  \\
     \approx & \frac{1}{n_d} \sum_{n} \left( \mathrm{KL}(q(s,a|x_n),p(s,a)) \right. \nonumber \\
     & \ \ \  \left. + \mathrm{KL}(q(z|x_n, d)),p(z|d)) \right) \overset{\Delta}{=} \mathcal{L}_{\mathrm{kl}}^d, \label{eq:kl}
\end{align}
with prior models $p(s,a)$ modeled as isotropic Gaussian and $p(z|d)$ as a two-layer MLP that takes the word embedding of $d \in \mathbb{R}^m$ as input.




The disentanglement of $z,s,a$ lies in their asymmetric roles in the generating processes. Specifically, among all latent variables $z,s,a$ that reconstructs $x$ via \emph{decoder} for $p_\theta(x|z,s,a)$, only $s,a$ are used in the \emph{classifier} for $p_\theta(y|s,a)$ and only $a$ is used in GCN for $p_\theta(A|a)$ to reconstruct attributes $A$. 

\noindent \textbf{Decoder for $p_{\theta}(x|z,s,a)$.} The decoder, that outputs $\hat{x} := \mathrm{Dec}(z,s,a)$, is trained to reconstruct the original image $x$ well by minimizing 
\begin{align}
    & \mathbb{E}_{p^d(x)}\left( \mathbb{E}_{q^d(v|x)}(\log{p_\theta(x|v)}) \right) \label{eq:dec} \\
    \approx & \frac{1}{n_d} \sum_n \Vert x_n - \hat{x}_n\Vert_2^2 \overset{\Delta}{=} \mathcal{L}_{\mathrm{rec}}^d. \nonumber 
\end{align}

\noindent\textbf{Graph Convolutional Network for $p_{\theta}(A|a)$.}
The correlations between attributes in lesions are strongly related and helpful to the cancer diagnosis. To make full use of this prior knowledge and further promote the disentanglement of the invariant disease-related features, we parameterize $p_{\theta}(A|a)$ by a Graph Convolutional Network (GCN) which is a flexible way to capture the topological structure in the label space.

Along with~\cite{chen2019multi}, we build a graph $G=(U, E)$ with twelve nodes and consider each attribute as a node, \emph{e.g.,} Shape-circle, Margin-clear. Each node $u \in U$ represents the word embedding of the attributes. Each edge $e \in E$ represents the inter-relevance between attributes. The inputs of the graph are features representations $H^l$ and corresponding correlation matrix $B$ which is calculated in the same as~\cite{chen2019multi}. For the first layer, $H^0 \in {\mathbb{R}}^{c \times c'}$ denotes the one-hot embedding matrix of each attribute node where $c$ is the number of attributes, $c'$ is the length of embeddings. Then, the feature representation of the graph at every layer~\cite{kipf2016semi} can be calculated as follow:
\begin{align}
    \label{eq:layer}
    H^{l+1} = \delta(BH^{l}W^{l}),
\end{align}
where $\delta(\cdot)$ is LeakyRelu, $W^{l}$ is the transformation matrix which is the parameter to be learned in the $l$th layer. The output $\{\hat{g}^k\}_{k}$ (with $\hat{g}^k:= \mathrm{GCN}([\text{Relevant-Enc}(x)]_{\mathcal{A}})$) is learned to approximate attributes $\{g^k\}_{k}$ (with each $g^k \in \{0,1\}$) to train the GCN: 
\begin{align}
    & - \mathbb{E}_{p^d(A|x)} \left( \log{p_{\theta}(A|x)} \right) \label{eq:gcn} \\
    \approx & -\frac{1}{n_d} \sum_{n} \left( \sum_{k=1}^C g_n^k \log \hat{g}_n^k + (1-g_n^k) \log (1-\hat{g}_n^k) \right) \overset{\Delta}{=} \mathcal{L}_{\mathrm{gcn}}^d.\nonumber 
\end{align}







\noindent\textbf{Disease Classifier for $p_{\theta}(y|s,a)$.} Finally, the classification network for $p_{\theta}(y|s,a)$ is trained via 
\begin{align}
    & - \mathbb{E}_{p^d(A|x)} \left( \log{p_{\theta}(y|A,x)} \right) \label{eq:cls} \\
    \approx & -\frac{1}{n_d} \sum_{n} \left( y_n \log\hat{y}_n+(1-y_n)\log(1-\hat{y}_n) \right) \overset{\Delta}{=} \mathcal{L}_{\mathrm{cls}}^d,
\end{align}
where $\hat{y}:=\mathrm{Cls}(\text{Relevant-Enc}(x))$ is the predicted label for $y$. 
Combined with Eq.~\eqref{eq:kl},~\eqref{eq:dec},~\eqref{eq:gcn},~\eqref{eq:cls}, the loss for domain $d$ is:
\begin{equation}
\label{eq:d}
    \mathcal{L}^d = \mathcal{L}_{\mathrm{kl}}^d +  \mathcal{L}_{\mathrm{rec}}^d + \mathcal{L}^d_{\mathrm{gcn}} +  \mathcal{L}_{\mathrm{cls}}^d.
\end{equation}

\subsubsection{\textbf{Variance Regularizer}}
\label{sec:var}
To further prompt the invariance for encoded $s,a$ across multiple domains, we leverage an additional \emph{variance regularizer} \cite{krueger2020out},
\begin{equation}
    \mathcal{L}_{\mathrm{var}} = \underset{d}{Var}\{\mathcal{L}^d_{\mathrm{gcn}}\} +
    \underset{d}{Var}\{\mathcal{L}^d_{\mathrm{cls}}\},
\end{equation}
where $\underset{d}{Var}\{\mathcal{L}^d_{\mathrm{gcn}}\}$ and $\underset{d}{Var}\{\mathcal{L}^d_{\mathrm{cls}}\}$ respectively denotes the variance of $\mathcal{L}^d_{\mathrm{gcn}}$ and $\mathcal{L}^d_{\mathrm{cls}}$ across $\mathcal{D}_{\mathrm{train}}$. This is motivated by our expectation to learn $s,a$ that can stably predict disease and reconstruct the clinical attributes. Combined with the loss for domain $d$ in Eq.~\eqref{eq:d}, the final loss is
\begin{align}
    \mathcal{L} = \sum_{d} \mathcal{L}^d +\beta \mathcal{L}_{\mathrm{var}}
\end{align}
where $\beta>0$ trades off the losses and the variance regularizer. 



\section{Experiments}

\noindent\textbf{Datasets.}
To evaluate the effectiveness of our model, we apply our model on patch-level mammogram mass benign/malignant classification, which drives increasing attention recently \cite{zhao20183d,lei2020shape} due to its clinical use. 
We consider both the public dataset DDSM~\cite{bowyer1996digital} and three in-house datasets: InH1, InH2 and InH3. These datasets are collected from different hospitals with different types of imaging devices (\emph{e.g.}, X-ray devices). 
For each dataset, the region of interests (ROIs) (malignant/benign masses) are cropped based on the annotations of radiologists the same as~\cite{kim2018icadx}. For each dataset, we randomly split it into training, validation, and testing sets with 8:1:1 patient-wise ratio\footnote{The \cite{bowyer1996digital} has not published their splitting method. We publish our test set for DDSM in Appendix}. The training/valid/testing samples we use contain 1165 ROIs from 571 patients/143 ROIs from 68 patients/147 ROIs from 75 patients in DDSM~\cite{bowyer1996digital}, 684 ROIs from 292 patients/87 ROIs from 38 patients/83 ROIs from 33 patients in InH1, 840 ROIs from 410 patients/104 ROIs from 50 patients/105 ROIs from 52 patients in InH2, and 565 ROIs from 271 patients/70 ROIs from 33 patients/70 ROIs from 34 patients in InH3.

\noindent\textbf{Implementation Details.} For a fair comparison, all methods are conducted under the same setting and share the same encoder backbone, \emph{i.e.}, ResNet34~\cite{he2016deep}. Meanwhile, the decoder is the deconvolution network of the encoder. For attribute annotations, in DDSM~\cite{bowyer1996digital} annotations can be parsed from the ".OVERLAY" file. The third line in the ".OVERLAY" file has annotations for types, shapes, and margins of masses. And in our in-house datasets, we obtain attribute annotations from the verify of one director doctor based on the annotations of three senior doctors. The inputs are resized into 224 × 224 with random horizontal flips and fed into networks. We implement all models with PyTorch. We implement Adam for optimization. The weight hyperparameter in \emph{variance regularizer} $\beta$ is 1 in our experiments. 
The clinical attributes contain \emph{circle, oval, irregular, circumscribed, obscured, ill-defined, is-lobulated, not-lobulated, is-spiculated, not-spiculated}. We add additional benign and malignant nodes to learn the correlation between the combination of attributes and benign/malignant. To verify the effectiveness of domain generalization, we test on all four datasets under OOD settings, \emph{i.e.}, every dataset is used as the testing set respectively while the other three are used as the training sets. To remove the randomness, we run for 10 times and report the average value of them. To further validate our effectiveness, we also give in-distribution results of each dataset which can be seen as the upper bounds of each dataset, \emph{i.e.}, using the training set with the same domain as the testing set. For a fair comparison, the number of above all training sets all keep the same. Area Under the Curve (AUC) is used as the evaluation metric in image-wise. For implementation of compared baselines, we directly load the published codes of ERM~\cite{he2016deep}, Chen \textit{et al.} \cite{chen2019multi}, DANN \cite{ganin2016domain}, MMD-AAE \cite{li2018domain}, DIVA \cite{ilse2020diva}, IRM \cite{arjovsky2019invariant} and Prithvijit \textit{et al.}\cite{chattopadhyay2020learning} during test; while we re-implement methods of Guided-VAE \cite{ding2020guided}, ICADx \cite{kim2018icadx} and Li \textit{et al.} \cite{li2019signed} for lacking published source codes.

\begin{table*}[!t]
{
\centering
\tiny
\resizebox{\linewidth}{!}{
\begin{tabular}{p{2.5cm}<{\centering}|p{1.2cm}<{\centering}|p{1.2cm}<{\centering}|p{1.2cm}<{\centering}|p{1.2cm}<{\centering}}
\hline
\multirow{2}{*}{Methodology} & \tabincell{c}{train on \\ InH2 \\ +InH3 \\ +DDSM} & \tabincell{c}{train on \\ InH1 \\ +InH3 \\ +DDSM} & \tabincell{c}{train on \\ InH1 \\ +InH2 \\ +DDSM} & \tabincell{c}{train on \\ InH1 \\ +InH2 \\ +InH3}  \\
\cline{2-5}
 & \tabincell{c}{test on \\ InH1} & \tabincell{c}{test on \\ InH2} & \tabincell{c}{test on \\ InH3} & \tabincell{c}{test on \\ DDSM}  \\
\hline
ERM~\cite{he2016deep} & 0.822 & 0.758 & 0.735 & 0.779 \\
Chen \textit{et al.} \cite{chen2019multi} & 0.877 & 0.827 & 0.804 & 0.830 \\
Guided-VAE \cite{ding2020guided} & 0.872 & 0.811 & 0.779 & 0.811 \\
ICADx \cite{kim2018icadx} & 0.882 & 0.802 & 0.777 & 0.826 \\
Li \textit{et al.} \cite{li2019signed} & 0.848 & 0.794 & 0.769 & 0.815 \\ 
DANN \cite{ganin2016domain} & 0.857 & 0.811 & 0.781 & 0.813 \\
MMD-AAE \cite{li2018domain} & 0.860 & 0.783 & 0.770 & 0.786 \\
DIVA \cite{ilse2020diva} & 0.865 & 0.809 & 0.784 & 0.813 \\
IRM \cite{arjovsky2019invariant} & 0.889 & 0.830 & 0.795 & 0.829 \\
Prithvijit \textit{et al.}\cite{chattopadhyay2020learning} & 0.851 & 0.796 & 0.772 & 0.797 \\
\cline{1-5}
\textbf{Ours} & \textbf{0.948} &  \textbf{0.874} & \textbf{0.858} & \textbf{0.892}\\
\hline
\end{tabular}
}
}
\caption{AUC evaluation of Domain Generalization on public DDSM~\cite{bowyer1996digital} and three in-house datasets. (OOD settings: training and testing are from different domains)}
\label{tab:comparison_all}
\end{table*}

\subsection{Results}

\noindent\textbf{Compared Baselines.} We compare our model with following methods: \textbf{a)} ERM~\cite{he2016deep} directly trains the classifier via ResNet34 by Empirical Risk Minimization; \textbf{b)} Chen \textit{et al.} \cite{chen2019multi} achieves multi-label classification with GCN for attributes prediction; \textbf{c)} Guided-VAE \cite{ding2020guided} also implements disentangle network but lacks the medical prior knowledge of attributes during learning; \textbf{d)} Li \textit{et al.} \cite{li2019signed} improve performance by generating more benign/malignant images via adversarial training; \textbf{e)} ICADx \cite{kim2018icadx} also proposes the adversarial learning method but additionally introduces shape/margins information for reconstruction; \textbf{f)} DANN~\cite{ganin2016domain} uses adversarial way to ensure a closer match between the source and the target distributions; \textbf{g)} MMD-AAE~\cite{li2018domain} extends adversarial autoencoders by imposing the Maximum Mean Discrepancy (MMD) measure; \textbf{h)} DIVA~\cite{ilse2020diva} proposes a generative model with three independent latent subspaces; \textbf{i)} IRM~\cite{arjovsky2019invariant} designs Invariant Risk Minimization term to learn invariant representation across environments; \textbf{j)} Prithvijit \textit{et al.}\cite{chattopadhyay2020learning} add domain specific masks learning for better domain generalization.


\begin{table*}[!t]
{
\centering
\tiny
\resizebox{\linewidth}{!}{
\begin{tabular}{p{2.4cm}<{\centering}|p{1.2cm}<{\centering}|p{1.2cm}<{\centering}|p{1.2cm}<{\centering}|p{1.2cm}<{\centering}}
\hline
\multirow{2}{*}{Methodology} & \tabincell{c}{train on \\ InH1} & \tabincell{c}{train on \\ InH2} & \tabincell{c}{train on \\ InH3} & \tabincell{c}{train on \\ DDSM~\cite{bowyer1996digital}}  \\
\cline{2-5}
 & \tabincell{c}{test on \\ InH1} & \tabincell{c}{test on \\ InH2} & \tabincell{c}{test on \\ InH3} & \tabincell{c}{test on \\ DDSM~\cite{bowyer1996digital}}  \\
\hline
ERM~\cite{he2016deep} & 0.888 & 0.847 & 0.776 & 0.847 \\
Chen \textit{et al.} \cite{chen2019multi} & 0.924 & 0.878 & 0.827 & 0.871 \\
Guided-VAE \cite{ding2020guided} & 0.921 & 0.867 & 0.809 & 0.869 \\
ICADx \cite{kim2018icadx} & 0.911 & 0.871 & 0.816 & 0.879 \\
Li \textit{et al.} \cite{li2019signed} & 0.908 & 0.859 & 0.828 & 0.875 \\ \cline{1-5}
\textbf{Ours-single} & \textbf{0.952} &  \textbf{0.898} & \textbf{0.864} & \textbf{0.919}\\
\hline
\end{tabular}
}
}
\caption{In-distribution AUC results on public DDSM~\cite{bowyer1996digital} and three in-house datasets (training and testing on the same single dataset.)}
\label{tab:comparison_all_single}
\end{table*}

\begin{table*}
{
\centering
\
\resizebox{\linewidth}{!}{
\begin{tabular}{c|c|c|c|c|c|c|c|c}
\hline
Irrelevant Encoder & Attribute Learning & $s$ & $\mathcal{L}_{rec}$ & Var & \textbf{AUC(InH1)} & \textbf{AUC(InH2)} & \textbf{AUC(InH3)} & \textbf{AUC(DDSM)}\\
\hline
$\times$ & $\times$ & $\times$ & $\times$ & $\times$ & 0.822 & 0.758 & 0.735 & 0.779\\
$\times$ & $\times$ & $\times$ & $\times$ & $\checkmark$ & 0.829 & 0.768 & 0.746 & 0.788\\
$\times$ & multi-task & $\times$ & $\times$ & $\times$ & 0.851 & 0.793 & 0.775 & 0.801\\
$\times$ & $\mathcal{L}_{gcn}$ & $\times$ & $\times$ & $\times$ & 0.877 & 0.827 & 0.804 & 0.830\\
$\times$ & $\mathcal{L}_{gcn}$ & $\times$ & $\times$ & $\checkmark$ & 0.884 & 0.835 & 0.808 & 0.835\\
$\times$ & $\mathcal{L}_{gcn}$ & $\times$ & $\checkmark$ & $\times$ & 0.911 & 0.846 & 0.816 & 0.844\\
DAL & $\mathcal{L}_{gcn}$ & $\times$ & $\checkmark$ & $\times$ & 0.931 & 0.862 & 0.841 & 0.878\\
DAL & $\times$ & $\times$ & $\checkmark$ & $\times$ & 0.913 & 0.840 & 0.823 & 0.852\\
$\times$ & $\mathcal{L}_{gcn}$ & $\checkmark$ & $\checkmark$ & $\times$ & 0.916 & 0.851 & 0.821 & 0.859\\
one branch & $\mathcal{L}_{gcn}$ & $\checkmark$ & $\checkmark$ & $\times$ & 0.926 & 0.856 & 0.830 & 0.863\\
DAL & $\mathcal{L}_{gcn}$ & $\checkmark$ & $\checkmark$ & $\times$ & 0.941 & 0.867 & 0.849 & 0.887\\
DAL & $\mathcal{L}_{gcn}$ & $\checkmark$ & $\checkmark$ & $\checkmark$ & \textbf{0.948} &  \textbf{0.874} & \textbf{0.858} & \textbf{0.892}\\
\hline
\end{tabular}
}
}
\caption{Ablation Studies Testing on InH1/InH2/InH3/DDSM~\cite{bowyer1996digital} while Training on the Other Three Datasets. (OOD settings)}
\label{ablation_all}
\end{table*}

\begin{table*}[t]
{
\centering
\tiny
\resizebox{\linewidth}{!}{
\begin{tabular}{p{1.8cm}<{\centering}|p{1.8cm}<{\centering}|p{1.8cm}<{\centering}|p{1.8cm}<{\centering}}
\hline
\tabincell{c}{train on \\ InH(1,2)} & \tabincell{c}{train on \\ InH(1,3)} & \tabincell{c}{train on \\ InH(3,2)} & \tabincell{c}{train on \\ InH(1,2,3)} \\
\hline
0.885 & 0.881 & 0.887 & 0.892 \\
\hline
\end{tabular}
}
}
\caption{Ablation study on the combination of training data sets. Take testing on pubilic dataset DDSM~\cite{bowyer1996digital} as an example. (OOD settings)}
\label{ablation_numdata}
\end{table*}

\begin{table*}[t]
{
\tiny
\centering
\resizebox{\linewidth}{!}{
\begin{tabular}{p{1.8cm}<{\centering}|p{1.8cm}<{\centering}|p{1.8cm}<{\centering}|p{1.8cm}<{\centering}}
\hline
\tabincell{c}{test on \\ InH1} & \tabincell{c}{test on \\ InH2} & \tabincell{c}{test on \\ InH3} & \tabincell{c}{test on \\ DDSM~\cite{bowyer1996digital}} \\
\hline
0.939 & 0.874 & 0.852 & 0.892 \\
\hline
\end{tabular}
}
}
\caption{AUC of testing on data set InH1/InH2/InH3/DDSM~\cite{bowyer1996digital} while training on InH1+InH2+InH3.}
\label{results_eachdata}
\end{table*}

\noindent\textbf{Results \& Analysis on Domain Generalization.} 
To verify the effectiveness of our learning method on out-of-distribution (OOD) samples, we train our model on the combination of three datasets from three different hospitals and test on the other unseen dataset from the other hospital which is the different domain from all training sets. 

As shown in Table~\ref{tab:comparison_all}, our methods can achieve state-of-the-art results in all settings. Specifically, the first five lines are the methods based on different representation learning and we extend them to our domain generalization task. The next five lines are the methods aiming at domain generalization. Li \textit{et al.} \cite{li2019signed} generate more data under the current domain, the larger number of data improves the performance compared with ERM~\cite{he2016deep} but the augmentation for the current domain greatly limits its ability of domain generalization. Prithvijit \textit{et al.}\cite{chattopadhyay2020learning} learn domain-specific mask (Clipart, Sketch, Painting), however, the gap exists in medical images can not balance through mask learning. DANN~\cite{ganin2016domain} and MMD-AAE~\cite{li2018domain} design distance constraints between the source and the target distributions. However, the key to achieving great classification performance in medical diagnosis is to explore the disease-related features which are invariant in multiple domains. Simply distance-constrain is not robust enough and limits the performance. The advantage of Guided-VAE~\cite{ding2020guided} and DIVA~\cite{ilse2020diva} over mentioned methods above may be due to the disentanglement learning in the former methods. IRM~\cite{arjovsky2019invariant} learns invariant representation across environments by Invariant Risk Minimization term which improves their results to a certain extent. However, lacking the guidance of attribute and disentanglement learning limits their performance. Guided-VAE~\cite{ding2020guided} introduces the attribute prediction which improves their performance than DIVA~\cite{ilse2020diva}. The improvements in ICADx~\cite{kim2018icadx}, Guided-VAE~\cite{ding2020guided} prove the importance of the guidance of attribute learning. Although ICADx~\cite{kim2018icadx} uses the attributes during learning, it fails to model correlations between attributes and benign/malignant diagnosis, which limits their performance. 
With further exploration of attributes via GCN, our method can outperform ICADx~\cite{kim2018icadx}, Guided-VAE~\cite{ding2020guided}. Compared to Chen \textit{et al.} \cite{chen2019multi} that also implement GCN to learn attributes, we additionally employ disentanglement learning with variance regularizer which can help to identify invariant disease-related features during prediction.

\noindent \textbf{Comparison with In-distribution results.} In addition, to further validate our effectiveness, we compute the in-distribution AUC performance of every single dataset. We implement the methods which aim at representation learning on each single dataset, \emph{i.e.}, training and testing on the data from the same hospital(domain). Such in-distribution results can serve as the upper bounds of our generalization method since their training and testing data come from the same domain distribution. To adapt our proposed mechanism to the in-distribution situation, we
change our network with two branches to only one branch accordingly for extracting features into $a,s,z$ since training data is only from one hospital(\textbf{Ours-single}), \emph{i.e.,} one domain without domain influence.
As shown in Table~\ref{tab:comparison_all_single}, based on disentanglement mechanism and the guidance of attribute learning, \textbf{Ours-single} still get the state-of-art performance. We argue that the disentangling mechanism with the guidance of attributes helps effective learning of disease-related features under a single domain. Results in Table~\ref{tab:comparison_all_single} can be seen as the upper bound results of each setting in Table~\ref{tab:comparison_all}. Our results in Table~\ref{tab:comparison_all} are slightly lower than results in Table~\ref{tab:comparison_all_single} by 0.4\% to 2.7\%. We argue that based on our mechanism for domain generalization, our method training under OOD can get the evenly matched effect with the training mode of the in-distribution that training and testing on the same domain. For example, as shown when testing on DDSM~\cite{bowyer1996digital}, performances of our model training on InH1+InH2+InH3 and training on DDSM itself are comparable.

\subsection{Ablation Study}

\noindent \textbf{Ablation study on each components.} To verify the effectiveness of each component in our model, we evaluate some variant models. Table~\ref{ablation_all} shows the ablation study results (under OOD settings: testing on InH1/InH2/InH3/DDSM respectively, training on the other three datasets).


Here are some interpretations for the variants:
\begin{enumerate}
\item \emph{Irrelevant Encoder} denotes whether using \emph{irrelevant encoder} during the reconstructing phase, with \emph{One branch} denotes only using one branch for the \emph{irrelevant encoder} without distinguishing multiple domains and \emph{DAL} denotes using domain adaptive layer for distinguishing multiple domains in \emph{irrelevant encoder};
\item \emph{Attribute Learning} denotes the way to use attributes: $\times$ means not using any attributes for learning, \emph{multi-task} means using a fully connected layer to predict the multiple attributes, and $\mathcal{L}_{gcn}$ means using our GCN network for learning attributes;
\item $s$ denotes whether split the latent factor $s$ out for disentanglement in training;
\item $\mathcal{L}_{rec}$ denotes whether use the reconstruction loss in training;
\item $Var$ denotes whether use the \emph{Variance Regularizer} in training.
\end{enumerate}

\begin{table*}[!t]
{
\centering
\tiny
\resizebox{\linewidth}{!}{
\begin{tabular}{p{2.5cm}<{\centering}|p{1.2cm}<{\centering}|p{1.2cm}<{\centering}|p{1.2cm}<{\centering}|p{1.2cm}<{\centering}}
\hline
\multirow{2}{*}{ratio of ADL} & \tabincell{c}{train on \\ InH2 \\ +InH3 \\ +DDSM} & \tabincell{c}{train on \\ InH1 \\ +InH3 \\ +DDSM} & \tabincell{c}{train on \\ InH1 \\ +InH2 \\ +DDSM} & \tabincell{c}{train on \\ InH1 \\ +InH2 \\ +InH3}  \\
\cline{2-5}
 & \tabincell{c}{test on \\ InH1} & \tabincell{c}{test on \\ InH2} & \tabincell{c}{test on \\ InH3} & \tabincell{c}{test on \\ DDSM}  \\
\hline
one layer & 0.926 & 0.857 & 0.835 & 0.864 \\
1/3 & 0.930 & 0.863 & 0.842 & 0.871 \\
1/2 & 0.942 & 0.871 & 0.847 & 0.883 \\
2/3 & 0.946 & 0.874 & 0.853 & 0.889 \\
\cline{1-5}
\textbf{all} & \textbf{0.948} &  \textbf{0.874} & \textbf{0.858} & \textbf{0.892}\\
\hline
\end{tabular}
}
}
\caption{Effectiveness of the ratio of domain adaptive layers on public DDSM~\cite{bowyer1996digital} and three in-house datasets.}
\label{tab:effectiveness_BN_number}
\end{table*}


\begin{table*}[!t]
{
\centering
\tiny
\resizebox{\linewidth}{!}{
\begin{tabular}{p{2.5cm}<{\centering}|p{1.2cm}<{\centering}|p{1.2cm}<{\centering}|p{1.2cm}<{\centering}|p{1.2cm}<{\centering}}
\hline
\multirow{2}{*}{Methodology} & \tabincell{c}{train on \\ InH2 \\ +InH3 \\ +DDSM} & \tabincell{c}{train on \\ InH1 \\ +InH3 \\ +DDSM} & \tabincell{c}{train on \\ InH1 \\ +InH2 \\ +DDSM} & \tabincell{c}{train on \\ InH1 \\ +InH2 \\ +InH3}  \\
\cline{2-5}
 & \tabincell{c}{test on \\ InH1} & \tabincell{c}{test on \\ InH2} & \tabincell{c}{test on \\ InH3} & \tabincell{c}{test on \\ DDSM}  \\
\hline
ME & 0.946 & 0.873 & 0.855 & 0.891 \\
GL & 0.947 & 0.872 & 0.854 & 0.887 \\
\cline{1-5}
\textbf{DAL} & \textbf{0.948} &  \textbf{0.874} & \textbf{0.858} & \textbf{0.892}\\
\hline
\end{tabular}
}
}
\caption{AUC evaluation of Domain Generalization on public DDSM~\cite{bowyer1996digital} and three in-house datasets under different domain adaptive mechanisms.}
\label{tab:comparison_adaptive_domain_layer}
\end{table*}

As shown, every component is effective for classiﬁcation performance. It is worth noting that using naive GCN also leads to a boosting of around 6\% in average. Such a result can demonstrate that the attributes can guide the learning of disease-related features. Meanwhile, disentanglement learning also causes a noticeable promotion, which may be due to that the disease-related features can be easier identified through disentanglement learning without mixing information with others. Moreover, Line7-8 in Table~\ref{ablation_all} validate disease-related features can be disentangled better with the guidance of exploring attributes. Line 2-3 from the bottom in Table~\ref{ablation_all} validates that distinguishing multiple domains improves the generalization performance. Comparing the last two lines, the regularizer we used is also helpful to disentangle invariant disease-related features. Besides Line2, 5 and 12 of Table~\ref{ablation_all} show that GCN and other components in our model are still well effective under variance constraints. 

To abate the impact of the combination of training domains, we train our model under different training combinations. Take testing on DDSM~\cite{bowyer1996digital} as an example. As shown in Table~\ref{ablation_numdata}, the more types of domains the better effect of our model. Due to the different correlations between different domains, the effect will be different under different combinations. But based on the inter mechanism of our model, influences between different domains are not obvious and three domains are sufﬁcient to achieve comparable results.

Under the setting: testing on DDSM~\cite{bowyer1996digital} (OOD) while training on InH1+InH2+InH3, we also list the results of our invariant model DIM-GCN (OOD model) under testing on the testing set of InH1/InH2/InH3 (in-distribution) as shown in Tab.~\ref{results_eachdata}. In addition, under the same setting, we also test our variant model \textbf{Ours-single} (in-distribution model). The result testing on unseen DDSM~\cite{bowyer1996digital} (OOD) is 0.861, testing on InH1/InH2/InH3 (in-distribution) which are from the same training sets (InH1+InH2+InH3) are 0.944, 0.880, and 0.853 respectively. The variant model testing on InH1/InH2/InH3 (the same domain as the training set) behaves comparably with ours in Tab.~\ref{results_eachdata} and is slightly better since our DIM-GCN split some inter-domain correlation which can decent performance under domain generalization. Thus, the variant model faces a larger drop over our invariant model DIM-GCN when generalizing to the unseen DDSM dataset (0.892 vs 0.861).

\begin{table*}[t]
{
\tiny
\centering
\resizebox{\linewidth}{!}{
\begin{tabular}{p{2.5cm}<{\centering}|p{1.2cm}<{\centering}|p{1.2cm}<{\centering}|p{1.2cm}<{\centering}|p{1.2cm}<{\centering}}
\hline
Methodology & \tabincell{c}{ACC \\-InH1} & \tabincell{c}{ACC \\-InH2} & \tabincell{c}{ACC \\-InH3} & \tabincell{c}{ACC \\ -DDSM} \\
\hline
ERM-multitask & 0.618 & 0.560 & 0.596 & 0.662\\
Chen \textit{et al.} \cite{chen2019multi} & 0.827 & 0.795 & 0.748 & 0.842\\
ICADX \cite{kim2018icadx} & 0.743 & 0.647 & 0.612 & 0.739 \\
Proposed Method & \textbf{0.914} & \textbf{0.877} & \textbf{0.858} & \textbf{0.934}\\
\hline
\end{tabular}
}
}
\caption{Overall Prediction Accuracy of Multi Attributes (Mass shapes, Mass margins). Testing on InH1/InH2/InH3/DDSM~\cite{bowyer1996digital} while Training on the Other Three Datasets. (OOD settings)}
\label{interpretable}
\end{table*}

\noindent \textbf{Ablation study on the ratio of using adaptive domain layers.} To verify the effectiveness of the ratio of using adaptive domain layers, we replaced the original BN layer with DAL in different ratios in the \emph{Irrelevant Encoder}. 
The results are shown in Tab.~\ref{tab:effectiveness_BN_number}, specifically, $1/3$ means only $1/3$ BN layers in the network are replaced, others and so forth. As we can see, under the lower ratio, the performances are close to \emph{One branch} in Tab.~\ref{ablation_all} for poorer domain-interpretability. The higher ratio can get better performance with more robust domain-interpretability.

\begin{figure*}[!t]
\begin{center}
    \includegraphics[height=0.61\linewidth]{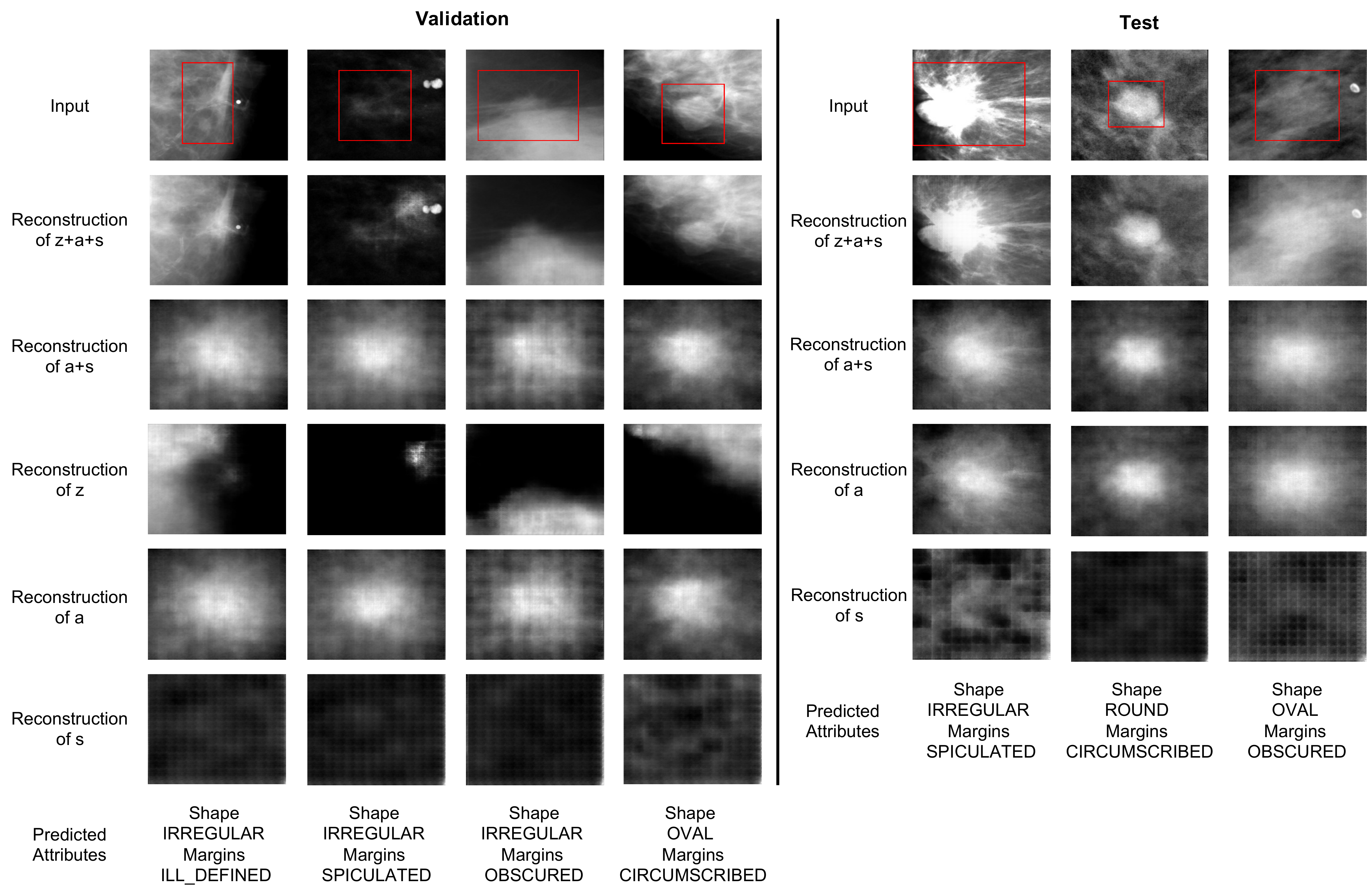}
\end{center}
    \centering
    \caption{Visualization on validation and testing dataset under OOD settings. Lesion regions are marked by red rectangles. Each row represents the reconstruction of different latent variables. Each column represents different cases. Note that there is no reconstruction result of $z$ at the test stage because the test domain has no corresponding irrelevant encoder.
}
    \label{fig:visual}
\end{figure*}

\noindent \textbf{Ablation study on Domain Adaptive Mechanism}
To investigate the proposed adaptive domain layer in the \emph{Irrelevant Encoder} sophisticatedly, we analyze various implementation forms of multiple domains as follows:

\begin{enumerate}
\item \textbf{Multiple Encoders(ME)}. Since the irrelevant encoder contains the information of domain environments, an intuitive idea is using multiple irrelevant encoders so as to each domain has one irrelevant encoder directly.
\item \textbf{Grouped Layer(GL)}. To reduce the parameter quantity of ME, we consider several groups of blocks with each group containing two blocks in the same structures. Each group only responds to one block each time, and different domains are different block combinations. The number of groups is set to $n$ that satisfies $2^n=m$ ($m$ denotes the number of domains, if $m$ is not the exponential power of 2, find $\hat{m}$ that is larger than $m$ and is the least number that satisfies $2^n=\hat{m}$). Thus each domain is a permutation combination based on each group choosing one block.

\item \textbf{Domain Adaptive Layer(DAL)}. To further reduce the parameter quantity and achieve domain generalization, we propose the domain adaptive layer. For the batch normalization layers in the irrelevant encoder, each domain contains one specific batch normalization layer. The scaling and shifting parameters in each layer are learned adaptively.
\end{enumerate}

We conduct experiments on the different implementation methods for modeling multiple domains. Results are shown in Table.~\ref{tab:comparison_adaptive_domain_layer}. Three different kinds of mechanisms have comparable performance. Since BN can usually be used as an effective measure for domain adaptation~\cite{ioffe2015batch}, \textbf{DAL} can be slightly better than the others with lighter computation, especially compared to \textbf{ME}.


\subsection{Prediction Accuracy of Attributes}
We argue that attributes can be the guidance of benign/malignant classification. In the current domain generalization task, we also calculate the OOD prediction accuracy of attributes in ours and other attribute-based representative methods in Table~\ref{interpretable}. The settings are: \textbf{ACC-InH1}: testing on InH1 while training on InH2+InH3+DDSM; \textbf{ACC-InH2}: testing on InH2 while training on InH1+InH3+DDSM; \textbf{ACC-InH3}: testing on InH3 while training on InH2+InH1+DDSM; \textbf{ACC-DDSM~\cite{bowyer1996digital}}: testing on DDSM while training on InH2+InH3+InH1. Our method gets the best prediction accuracy on the attributes over other methods under OOD.

\subsection{Visualization}

We investigate the three latent factors $a$, $s$, and $z$ to see if our model is able to successfully disentangle them. We visualize different parts in Fig.~\ref{fig:visual} via the reconstruction effect and the predicted attributes of the current image. Since the \emph{Irrelevant Enc} is partially domain-dependent, validating set is from the same domain as the training set but the testing set is from a different domain. We show reconstructions of all latent factors in validation (Left in Fig.~\ref{fig:visual}) and reconstructions without $z$ in test (Right in Fig.~\ref{fig:visual}). As we can see, the disease-related features $s+a$ mainly reflect the disease-related information since they mainly reconstruct the lesion regions without mixing others. The disease-irrelevant $z$ features mainly learn features such as the contour of the breasts, pectoralis, and other irrelevant glands without lesion information. It is worth noting that the white dots on the image which are caused by machine shooting are learned by $z$ as visualization. This means that through the ability of domain generalization our method can disentangle the irrelevant part successfully and prevent it from predicting the disease. Moreover, the macroscopic features $a$ capture the macroscopic attributes of the lesions, \emph{e.g.,} shape and density; while the microscopic features $s$ learn properties like global context, texture, or other invisible features but related to disease classification. These results further indicate the effectiveness and interpretability of our DIM-GCN.

\section{Conclusion}

We propose a novel Domain Invariant Model with Graph Convolutional Network (DIM-GCN) on domain generalization for medical diagnosis, which is a multi-domain variational auto-encoder with the disentanglement mechanism equipped with a graph convolutional network. We evaluate our method on both public and in-house datasets for patch-level mammogram mass benign/malignant classification. Potential results demonstrate the effectiveness of our DIM-GCN, we will try to generalize this method to other medical imaging problems such as lung cancer, liver cancer, \emph{etc}.


%

\newpage

\appendices

\section{Formal Description of Theorem 3.1}

We assume that the $s,a|y$ and $z|d$ belong to the following exponential families:
\begin{align}
    p(z|d) & := p_{\mathbf{T}^z,\bm{\Gamma}^z_{d}}(z|d), \nonumber \\
p(s,a|y) & := p_{\mathbf{T}^s, \bm{\Gamma}^s_{y}}(s|y)p_{\mathbf{T}^a,\bm{\Gamma}^a_y}(a|y), \nonumber
\end{align}
where 
\begin{equation*}
    p_{\mathbf{T}^{u}, \bm{\Gamma}^u_{o}}(u) = {\small\prod_{i=1}^{q_u}} \exp\Big( {\small \sum_{j=1}^{k_u}} T^u_{i,j}(t_i) \Gamma^t_{o,i,j} + B_i(u_i) -  C^u_{o,i} \Big),
\end{equation*}
for any $u \in \{s,a\}$ with $o=y$ and $u = z$ with $o=d$. The $\{T^u_{i,j}(u_i)\}$, $\{\Gamma^u_{o,i,j}\}$ denote the sufficient statistics and natural parameters, $\{B_i\}$ and $\{C_{o,i}^u\}$ denote the base measures and normalizing constants to ensure the integral of distribution equals to 1. Let {\small $\mathbf{T}^u(u)\!:=\![\mathbf{T}^{u}_{1}(u_{1}),...,\mathbf{T}^{u}_{q_u}(u_{q_{u}}) ]$ $\!\in\! \mathbb{R}^{k_u \times q_u}$ $\big(\mathbf{T}^u_{i}(u_{i}) \!:=\! [T^u_{i,1}(u_i),...,T^u_{i,k_u}(u_i)], \forall i \in [q_u]\big)$}, {\small$\bm{\Gamma}^u_{o} \!:=\! \left[\bm{\Gamma}^{u}_{o,1},...,\bm{\Gamma}^{u}_{o,q_u} \right]$ $\!\in\! \mathbb{R}^{k_u \times q_u}$ $\big(\bm{\Gamma}^u_{o,i} \!:=\! [\Gamma^t_{o,i,1},...,\Gamma^t_{o,i,k_u}], \forall i \in [q_u]\big)$}. Further, we assume that the $s,a,z \to x$ and $a \to A$ follow the \emph{additive noise model} (ANM): 
\begin{align}
    x = f_x(s,a,z) + \epsilon_x, A = f_A(a) + \epsilon_A. \nonumber 
\end{align}
Denote $\theta:=\{\mathbf{T}^s,\mathbf{T}^z,\mathbf{T}^a,\bm{\Gamma}^s_{y},\bm{\Gamma}^a_{y},\bm{\Gamma}^z_{d},f_x,f_A\}$, we define the \emph{disentanglement} as follows: 
\begin{definition}
\label{def:dis}
We call the $s,a,z$ as disentangled under $\theta$, if for any $\tilde{\theta}:=\{\tilde{\mathbf{T}}^s,\tilde{\mathbf{T}}^z,\tilde{\mathbf{T}}^a,\tilde{\bm{\Gamma}}^s_{y},\tilde{\bm{\Gamma}}^a_{y},\tilde{\bm{\Gamma}}^z_{d},\tilde{f}_x,\tilde{f}_A \}$ that giving rise to the same observational distributions: $p_\theta(x,A,y|d) = p_{\tilde{\theta}}(x,A,y|d)$ for any $x,y,A$ and $d$, there exists invertible matrices $M_s,M_z,M_a$ and vectors $b_z,b_s,b_a$ such that:
\begin{align}
    \tilde{\mathbf{T}}([\tilde{f}_x^{-1}]_{\mathcal{S}}(x)) & = M_s\mathbf{T}([f_x^{-1}]_{\mathcal{S}}(x)) + b_s, \nonumber \\
    \tilde{\mathbf{T}}([\tilde{f}_x^{-1}]_{\mathcal{Z}}(x)) & = M_z\mathbf{T}([f_x^{-1}]_{\mathcal{Z}}(x)) + b_z, \nonumber \\
    \tilde{\mathbf{T}}([\tilde{f}_x^{-1}]_{\mathcal{A}}(x)) & = M_a\mathbf{T}([f_x^{-1}]_{\mathcal{A}}(x)) + b_a, \nonumber
\end{align}
where the $\mathcal{S},\mathcal{Z},\mathcal{A}$ denote the space of the latent variables $s,z,a$. 
\end{definition}
The theorem 3.1 is then mathematically formulated as:
\begin{theorem}
\label{thm:disentangle}
For any $\theta$, under following assumptions:
\begin{enumerate}[nosep]
    \item The characteristic functions of $\epsilon_x,\epsilon_A$ are almost everywhere nonzero.
    \item $f_x, f_A$ are bijective;
    \item The $\{T^u_{i,j}\}_{1\leq j \leq k_u}$ are linearly independent in $\mathcal{S}$, $\mathcal{Z}$ or $\mathcal{A}$ for each $i \in [q_u]$ for any $u=s,a,z$.
    \item There exists $d_1,...,d_m$ and $y_1,...,y_K$ such that $\big[[\bm{\Gamma}^{z}_{d_2} - \bm{\Gamma}^{z}_{d_1}]^\mathsf{T},...,[\bm{\Gamma}^{z}_{d_m} - \bm{\Gamma}^{z}_{d_1}]^\mathsf{T}\big]^\mathsf{T} \in \mathbb{R}^{m \times (q_z \times k_z)}$ and $\big[[\bm{\Gamma}^{u=s,a}_{y_2} - \bm{\Gamma}^{u=s,a}_{y_1}]^\mathsf{T},...,[\bm{\Gamma}^{u=s,a}_{y_K} - \bm{\Gamma}^{u=s,a}_{y_1}]^\mathsf{T}\big]^\mathsf{T} \in \mathbb{R}^{m \times (q_u \times k_u)}$ have full column rank.
\end{enumerate}
we have that $s,z,a$ are disentangled under $\theta$.
\end{theorem}

\begin{proof}
For simplicity, we denote $\!\tilde{p}(u|o):= p_{\tilde{\mathbf{T}}^u,\tilde{\bm{\Gamma}}^u_o}(u|o)$. Since $p_\theta(x|d,y)=p_{\tilde{\theta}}(x|d,y)$, then we have $\int p_{f_x}(x|s,a,z)p(s,a|y)p(z|d)dsdadz \!=\! \int p_{\tilde{f}_x}(x|s,a,z) \! \tilde{p}(s,a|y)\tilde{p}(z|d)dsdadz$. According to the chain rule of changing from $s,a,z$ to $\bar{x} \!:=\! f_x(s,a,z)$, we have that $\int p_{\varepsilon_x}(x-\bar{x})p(f_x^{-1}(\bar{x})|d,y)J_{f^{-1}}(\bar{x})d\bar{x} \!=\! \int p_{\varepsilon_x}(x-\bar{x})p(\tilde{f}_x^{-1}(\bar{x})|d,y)J_{\tilde{f}^{-1}}(\bar{x})d\bar{x}$, where $J_f(x)$ denotes the Jacobian matrix of $f$ on $x$. Denote $p'(\bar{x}|d,y):=p(f_x^{-1}(\bar{x})|d,y)J_{f^{-1}}(\bar{x})$. Applying Fourier transformation to both sides, we have $F[p'](\omega)\varphi_{\varepsilon_x}(\omega)=F[\tilde{p}'](\omega)\varphi_{\varepsilon_x}(\omega)$, where $\varphi_{\varepsilon_x}$ denotes the characteristic function of $\varepsilon_x$. Since they are almost everywhere nonzero, we have that $F[p'](\omega)=F[\tilde{p}']$, which means that $p'(\bar{x}|d,y)=\tilde{p}'(\bar{x}|d,y)$. This is equivalent to the following:
\begingroup
\allowdisplaybreaks
\begin{align}
\label{eq:log-x}
& \log{\mathrm{vol}J_{f_x}(x)} + \sum_{u=s,a}\sum_{i=1}^{q_u}( \log{B_i([f_{x,i}^{-1}]_{\mathcal{U}}(x)}) - \log{C_i(y)} \nonumber \\
& + \sum_{j=1}^{k_u} T^u_{i,j}(f_{x,i}^{-1}(x))\Gamma^u_{i,j}(x)) + \sum_{i=1}^{q_z}( \log{B_i([f_{x,i}^{-1}]_{\mathcal{Z}}(x)}) \nonumber \\
& - \log{C_i(d)} + \sum_{j=1}^{k_z} T^z_{i,j}(f_{x,i}^{-1}(x))\Gamma^z_{i,j}(x)) \nonumber \\
& = \log{\mathrm{vol}J_{\tilde{f}_x}(x)} + \sum_{u=s,a}\sum_{i=1}^{q_u}( \log{\tilde{B}_i([\tilde{f}_{x,i}^{-1}]_{\mathcal{U}}(x)}) - \log{\tilde{C}_i(y)} \nonumber \\
& + \sum_{j=1}^{k_u} \tilde{T}^u_{i,j}(\tilde{f}_{x,i}^{-1}(x))\tilde{\Gamma}^u_{i,j}(x)) + \sum_{i=1}^{q_z}( \log{\tilde{B}_i([\tilde{f}_{x,i}^{-1}]_{\mathcal{Z}}(x)}) \nonumber \\
& - \log{\tilde{C}_i(d)} + \sum_{j=1}^{k_z} \tilde{T}^z_{i,j}(\tilde{f}_{x,i}^{-1}(x))\tilde{\Gamma}^z_{i,j}(x)). 
\end{align}
\endgroup
Subtract the Eq.~\eqref{eq:log-x} from the one with $y_k$, we have that 
\begingroup
\allowdisplaybreaks
\begin{align}
\label{eq:sa-y}
	& \sum_{u=s,a} \left( \langle \mathbf{T}^u([f_x^{-1}]_{\mathcal{U}}(x)),\bm{\overline{\Gamma}}^u(y_k)\rangle + \sum_i \log{ \frac{C_i(y_1)}{C_i(y_k)} } \right) \nonumber \\
	& = \sum_{u=s,a} \left( \langle \tilde{\mathbf{T}}^u([\tilde{f}_x^{-1}]_{\mathcal{U}}(x)),\bm{\overline{\tilde{\Gamma}}}^u(y_k)\rangle + \sum_i \log{\frac{\tilde{C}_i(y_1)}{\tilde{A}_i(y_k)} } \right),
\end{align}
\endgroup
for all $k \in [m]$, where $\bm{\bar{\Gamma}}(y) = \bm{\Gamma}(y) - \bm{\Gamma}(y_1)$. Denote $\tilde{b}_u(k) = \sum_{u=s,a}\sum_{i}^{q_u} \frac{\tilde{A}^u_i(d_k)A^u_i(d_k)}{\tilde{C}^u_i(d_k)A^u_i(d_1)}$ for $k \in [K]$. Besides, by subtracting the Eq.~\eqref{eq:log-x} from the one with $d_l$, we have 
\begingroup
\allowdisplaybreaks
\begin{align}
\label{eq:z-d}
	&  \langle \mathbf{T}^z([f_x^{-1}]_{\mathcal{Z}}(x)),\bm{\overline{\Gamma}}^z(d_l)\rangle + \sum_i \log{ \frac{C_i(d_1)}{C_i(d_l)} } \nonumber \\
	& = \langle \tilde{\mathbf{T}}^z([\tilde{f}_x^{-1}]_{\mathcal{Z}}(x)),\bm{\overline{\tilde{\Gamma}}}^z(d_l)\rangle + \sum_i \log{\frac{\tilde{C}_i(d_1)}{\tilde{A}_i(d_l)} },
\end{align}
\endgroup
for all $k \in [m]$, where $\bm{\bar{\Gamma}}(d) = \bm{\Gamma}(d) - \bm{\Gamma}(d_1)$. Denote $\tilde{b}_z(l) = \sum_i \frac{\tilde{A}^z_i(d_l)A^s_i(d_l)}{\tilde{A}^z_i(d_l)A^z_i(d_1)}$ for $l \in [m]$. According to assumption (4), we have that:
\begingroup
\allowdisplaybreaks
\begin{align}
& \bm{\overline{\Gamma}}^{z,\top} \mathbf{T}^z([f_x^{-1}]_{\mathcal{Z}}(x)) = \bm{\overline{\tilde{\Gamma}}}^{z,\top} \tilde{\mathbf{T}}^z([\tilde{f}_x^{-1}]_{\mathcal{Z}}(x)) + \tilde{b}_z, \label{eq:z-inv}\\ 
& \bm{\overline{\Gamma}}^{s,\top} \mathbf{T}^s([f_x]_{\mathcal{S}}^{-1}(x)) + \bm{\overline{\Gamma}}^{a,\top} \mathbf{T}^a([f_x]_{\mathcal{A}}^{-1}(x)) \nonumber \\
& = \bm{\overline{\tilde{\Gamma}}}^{s,\top} \tilde{\mathbf{T}}^s([\tilde{f}_x]_{\mathcal{S}}^{-1}(x)) + \bm{\overline{\tilde{\Gamma}}}^{a,\top} \tilde{\mathbf{T}}^a([\tilde{f}_x]_{\mathcal{A}}^{-1}(x)) + \tilde{b}_a + \tilde{b}_s.\label{eq:sa-inv}
\end{align}
\endgroup
Similarly, we also have $p'(\bar{A}|y)=\tilde{p}'(\bar{A}|y)$, which means that
\begingroup
\allowdisplaybreaks
\begin{align}
\label{eq:log-A}
& \log{\mathrm{vol}J_{f_A}(A)} + \sum_{i=1}^{q_a}( \log{B_i([f_{A,i}^{-1}]_{\mathcal{A}}(A)}) - \log{C_i(d)} \nonumber \\
& + \sum_{j=1}^{k_a} T^a_{i,j}(f_{A,i}^{-1}(A))\Gamma^a_{i,j}(x)) \nonumber \\
& = \log{\mathrm{vol}J_{\tilde{f}_A}(A)} + \sum_{i=1}^{q_a}( \log{B_i([\tilde{f}_{A,i}^{-1}]_{\mathcal{A}}(A)}) - \log{C_i(d)} \nonumber \\
& + \sum_{j=1}^{k_a} \tilde{T}^a_{i,j}(\tilde{f}_{A,i}^{-1}(A))\tilde{\Gamma}^a_{i,j}(A)),
\end{align}
\endgroup
which implies that 
\begingroup
\allowdisplaybreaks
\begin{align}
\bm{\overline{\Gamma}}^{a,\top} \mathbf{T}^a([f_A^{-1}]_{\mathcal{A}}(A)) = \bm{\overline{\tilde{\Gamma}}}^{a,\top} \tilde{\mathbf{T}}^a([\tilde{f}_A^{-1}]_{\mathcal{A}}(A)) + \tilde{b}_a.
\end{align}
\endgroup
Denote $v := [x^{\top},A^{\top}]^{\top}$, $\epsilon := [\epsilon_x^{\top},\epsilon_A^{\top}]^{\top}$, $h(v) = [[f_x]_{\mathcal{S}}^{-1}(x)^{\top},f_A^{-1}(A)^{\top}]^{\top}$. Applying the same trick above, we have that 
\begingroup
\allowdisplaybreaks
\begin{align}
    \bm{\overline{\Gamma}}^{a,\top} \mathbf{T}^a([f_x]_{\mathcal{A}}^{-1}(x)) = \bm{\overline{\tilde{\Gamma}}}^{a,\top} \tilde{\mathbf{T}}^a([\tilde{f}_x]_{\mathcal{A}}^{-1}(x))  + \tilde{b}_a. \label{eq:a-inv}
\end{align}
\endgroup
Combining Eq.~\eqref{eq:z-inv},~\eqref{eq:sa-inv},~\eqref{eq:a-inv}, we have that 
\begingroup
\allowdisplaybreaks
\begin{align}
  \bm{\overline{\Gamma}}^{z,\top} \mathbf{T}^z([f_x^{-1}]_{\mathcal{Z}}(x)) = \bm{\overline{\tilde{\Gamma}}}^{z,\top} \tilde{\mathbf{T}}^z([\tilde{f}_x^{-1}]_{\mathcal{Z}}(x)) + \tilde{b}_z,   \\
  \bm{\overline{\Gamma}}^{a,\top} \mathbf{T}^a([f_x]_{\mathcal{A}}^{-1}(x)) = \bm{\overline{\tilde{\Gamma}}}^{a,\top} \tilde{\mathbf{T}}^a([\tilde{f}_x]_{\mathcal{A}}^{-1}(x))  + \tilde{b}_a. \\
  \bm{\overline{\Gamma}}^{s,\top} \mathbf{T}^s([f_x]_{\mathcal{S}}^{-1}(x)) = \bm{\overline{\tilde{\Gamma}}}^{s,\top} \tilde{\mathbf{T}}^s([\tilde{f}_x]_{\mathcal{S}}^{-1}(x))  + \tilde{b}_s.
\end{align}
\endgroup
Applying the same trick in \cite[Theorem 7.9]{sun2020latent}, we have that $(\bm{\overline{\Gamma}}^{u,\top})^{-1}\bm{\overline{\tilde{\Gamma}}}^{u,\top}$ are invertible for $u=s,a,z$. 
\end{proof}

\section{Test Set of DDSM}
To provide convenience for latter works, we publish the list of our test division on the public dataset DDSM~\cite{bowyer1996digital}.

\begin{lstlisting}
benign_12_case1889 benign_04_case3357 
cancer_01_case3084 benign_04_case0304 
benign_09_case4060 benign_05_case1491
cancer_08_case1464 cancer_09_case0049 
cancer_11_case1678 cancer_04_case1090 
cancer_05_case0157 benign_06_case0366
benign_04_case0270 benign_02_case1321 
cancer_05_case0142 cancer_05_case0127 
benign_04_case3103 cancer_07_case1143 
cancer_08_case1128 benign_11_case1792 
benign_06_case0396 cancer_15_case3371 
benign_07_case1686 benign_13_case0485
benign_09_case4085 cancer_02_case0112 
cancer_15_case3398 benign_03_case1435 
cancer_01_case3027 cancer_07_case1114
cancer_03_case1070 benign_03_case1432 
cancer_06_case1182 cancer_05_case0140 
benign_12_case1947 benign_12_case1922
cancer_05_case0210 cancer_08_case1403 
cancer_05_case0173 benign_01_case0235 
benign_02_case1317 benign_11_case1836
cancer_05_case0222 cancer_08_case1532 
benign_06_case0372 cancer_02_case0077 
benign_11_case1855 cancer_05_case0139
benign_08_case1786 cancer_07_case1159 
cancer_10_case1573 cancer_05_case0181 
benign_09_case4038 cancer_05_case0192
benign_06_case0363 cancer_06_case1122 
benign_01_case3113 benign_09_case4003 
benign_06_case0367 cancer_12_case4139
cancer_14_case1985 cancer_05_case0183 
cancer_10_case1642 cancer_05_case0206 
cancer_03_case1007 cancer_12_case4108
cancer_09_case0340 benign_07_case1412 
cancer_05_case0085 benign_09_case4065 
benign_03_case1363 benign_09_case4027
benign_10_case4016 benign_13_case3433 
benign_09_case4090
\end{lstlisting}


\ifCLASSOPTIONcaptionsoff
  \newpage
\fi



%
{\small
\bibliographystyle{plain}
\bibliography{egbib}

\begin{thebibliography}{10}

\bibitem{arjovsky2019invariant}
Martin Arjovsky, L{\'e}on Bottou, Ishaan Gulrajani, and David Lopez-Paz.
\newblock Invariant risk minimization.
\newblock {\em arXiv preprint arXiv:1907.02893}, 2019.

\bibitem{azulay2018deep}
Aharon Azulay and Yair Weiss.
\newblock Why do deep convolutional networks generalize so poorly to small
  image transformations?
\newblock {\em arXiv preprint arXiv:1805.12177}, 2018.

\bibitem{bowyer1996digital}
K~Bowyer, D~Kopans, WP~Kegelmeyer, R~Moore, M~Sallam, K~Chang, and K~Woods.
\newblock The digital database for screening mammography.
\newblock In {\em Third international workshop on digital mammography},
  volume~58, page~27, 1996.

\bibitem{bowyer1996digitalshort}
K~Bowyer~et al.
\newblock The digital database for screening mammography.
\newblock In {\em Third international workshop on digital mammography},
  volume~58, page~27, 1996.

\bibitem{chattopadhyay2020learning}
Prithvijit Chattopadhyay, Yogesh Balaji, and Judy Hoffman.
\newblock Learning to balance specificity and invariance for in and out of
  domain generalization.
\newblock In {\em European Conference on Computer Vision}, pages 301--318.
  Springer, 2020.

\bibitem{chen2019multi}
Zhao-Min Chen, Xiu-Shen Wei, Peng Wang, and Yanwen Guo.
\newblock Multi-label image recognition with graph convolutional networks.
\newblock In {\em Proceedings of the IEEE/CVF Conference on Computer Vision and
  Pattern Recognition}, pages 5177--5186, 2019.

\bibitem{ding2020optimizing}
Jie Ding, Shenglan Chen, Mario~Serrano Sosa, Renee Cattell, Lan Lei, Junqi Sun,
  Prateek Prasanna, Chunling Liu, and Chuan Huang.
\newblock Optimizing the peritumoral region size in radiomics analysis for
  sentinel lymph node status prediction in breast cancer.
\newblock {\em Academic Radiology}, 2020.

\bibitem{ding2020guided}
Zheng Ding, Yifan Xu, Weijian Xu, Gaurav Parmar, Yang Yang, Max Welling, and
  Zhuowen Tu.
\newblock Guided variational autoencoder for disentanglement learning.
\newblock In {\em Proceedings of the IEEE/CVF Conference on Computer Vision and
  Pattern Recognition}, pages 7920--7929, 2020.

\bibitem{ganin2016domain}
Yaroslav Ganin, Evgeniya Ustinova, Hana Ajakan, Pascal Germain, Hugo
  Larochelle, Fran{\c{c}}ois Laviolette, Mario Marchand, and Victor Lempitsky.
\newblock Domain-adversarial training of neural networks.
\newblock {\em The journal of machine learning research}, 17(1):2096--2030,
  2016.

\bibitem{he2016deep}
Kaiming He, Xiangyu Zhang, Shaoqing Ren, and Jian Sun.
\newblock Deep residual learning for image recognition.
\newblock In {\em Proceedings of the IEEE conference on computer vision and
  pattern recognition}, pages 770--778, 2016.

\bibitem{ilse2020diva}
Maximilian Ilse, Jakub~M Tomczak, Christos Louizos, and Max Welling.
\newblock Diva: Domain invariant variational autoencoders.
\newblock In {\em Medical Imaging with Deep Learning}, pages 322--348. PMLR,
  2020.

\bibitem{ioffe2015batch}
Sergey Ioffe and Christian Szegedy.
\newblock Batch normalization: Accelerating deep network training by reducing
  internal covariate shift.
\newblock In {\em International conference on machine learning}, pages
  448--456. PMLR, 2015.

\bibitem{kim2018icadx}
Seong~Tae Kim, Hakmin Lee, Hak~Gu Kim, and Yong~Man Ro.
\newblock Icadx: interpretable computer aided diagnosis of breast masses.
\newblock In {\em Medical Imaging 2018: Computer-Aided Diagnosis}, volume
  10575, page 1057522. International Society for Optics and Photonics, 2018.

\bibitem{kingma2013auto}
Diederik~P Kingma and Max Welling.
\newblock Auto-encoding variational bayes.
\newblock {\em arXiv preprint arXiv:1312.6114}, 2013.

\bibitem{kipf2016semi}
Thomas~N Kipf and Max Welling.
\newblock Semi-supervised classification with graph convolutional networks.
\newblock {\em arXiv preprint arXiv:1609.02907}, 2016.

\bibitem{krueger2020out}
David Krueger, Ethan Caballero, Joern-Henrik Jacobsen, Amy Zhang, Jonathan
  Binas, Dinghuai Zhang, Remi~Le Priol, and Aaron Courville.
\newblock Out-of-distribution generalization via risk extrapolation (rex).
\newblock {\em arXiv preprint arXiv:2003.00688}, 2020.

\bibitem{lei2020shape}
Yiming Lei~et al.
\newblock Shape and margin-aware lung nodule classification in low-dose ct
  images via soft activation mapping.
\newblock {\em Medical image analysis}, 60:101628, 2020.

\bibitem{li2018domain}
Haoliang Li, Sinno~Jialin Pan, Shiqi Wang, and Alex~C Kot.
\newblock Domain generalization with adversarial feature learning.
\newblock In {\em Proceedings of the IEEE Conference on Computer Vision and
  Pattern Recognition}, pages 5400--5409, 2018.

\bibitem{li2019signed}
Heyi Li, Dongdong Chen, William~H Nailon, Mike~E Davies, and David~I Laurenson.
\newblock Signed laplacian deep learning with adversarial augmentation for
  improved mammography diagnosis.
\newblock In {\em International Conference on Medical Image Computing and
  Computer-Assisted Intervention}, pages 486--494. Springer, 2019.

\bibitem{sicklesacr}
E~Sickles, CJ~D’Orsi, and LW~Bassett.
\newblock Acr bi-rads{\textregistered} mammography. acr
  bi-rads{\textregistered} atlas, breast imaging reporting and data system.
  american college of radiology 2013.

\bibitem{sun2020latent}
Xinwei Sun, Botong Wu, Chang Liu, Xiangyu Zheng, Wei Chen, Tao Qin, and Tie-yan
  Liu.
\newblock Latent causal invariant model.
\newblock {\em arXiv preprint arXiv:2011.02203}, 2020.

\bibitem{surendiran2012mammogram}
B~Surendiran and A~Vadivel.
\newblock Mammogram mass classification using various geometric shape and
  margin features for early detection of breast cancer.
\newblock {\em International Journal of Medical Engineering and Informatics},
  4(1):36--54, 2012.

\bibitem{zhao20183d}
Wei Zhao~et al.
\newblock 3d deep learning from ct scans predicts tumor invasiveness of
  subcentimeter pulmonary adenocarcinomas.
\newblock {\em Cancer research}, 78(24):6881--6889, 2018.

\end{thebibliography}
}

%




\end{document}